\documentclass[10pt]{article}

\usepackage[preprint]{tmlr}

\usepackage{amsmath,amssymb,amsthm}
\usepackage{mathtools}
\usepackage{graphicx}
\usepackage{booktabs}
\usepackage{enumitem}
\usepackage{algorithm}
\usepackage{algpseudocode}
\usepackage{microtype}
\usepackage{hyperref}
\usepackage{url}

\theoremstyle{plain}
\newtheorem{theorem}{Theorem}
\newtheorem{proposition}[theorem]{Proposition}
\newtheorem{corollary}[theorem]{Corollary}
\newtheorem{lemma}[theorem]{Lemma}
\theoremstyle{definition}
\newtheorem{definition}[theorem]{Definition}
\newtheorem{remark}[theorem]{Remark}
\newtheorem{conjecture}[theorem]{Conjecture}

\newcommand{\R}{\mathbb{R}}
\newcommand{\norm}[1]{\left\|#1\right\|}
\newcommand{\E}{\mathbb{E}}
\newcommand{\dH}{d_H}
\newcommand{\Pmds}{P_{\mathrm{MDS}}}
\newcommand{\softmax}{\operatorname{softmax}}

\title{On the Geometry of Positional Encodings in Transformers}

\author{Giansalvo Cirrincione \\
  Laboratoire LTI, Université de Picardie Jules Verne \\
  Chemin du Thil, 80025 Amiens, France \\
  \texttt{giansalvo.cirrincione@u-picardie.fr}}

\date{}

\begin{document}
\maketitle

\begin{center}
\textit{Submitted to Transactions on Machine Learning Research (TMLR)}
\end{center}

\vspace{0.5em}

\begin{abstract}
Neural language models process sequences of words, but the mathematical
operations inside them --- matrix multiplications and attention mechanisms
--- are insensitive to the order in which words appear. Positional encodings
are the component added to remedy this: they inject information about the
position of each word into its vector representation. Despite their
importance, positional encodings have been designed largely by trial and
error, without a mathematical theory of what they ought to do.

This paper develops such a theory. Three questions are addressed. First,
is positional information strictly necessary? It is proved that any
Transformer without a positional signal treats every permutation of the
input as equivalent to the original, and therefore cannot solve any task
sensitive to word order (Theorem~\ref{thm:necessity}). Second, what
structure does a learned positional encoding acquire? The \emph{Positional
Separation Theorem} (Theorem~\ref{thm:separation}) establishes that, under
mild and verifiable conditions, training assigns distinct vector
representations to distinct sequence positions at every global minimiser.
Third, what would an optimal positional encoding look like? Each position
in a corpus has a characteristic distribution of words that tend to appear
there; the natural criterion for an encoding is to reproduce the statistical
distances between these distributions. An exact reproduction is shown to be
impossible in general (the relevant geometry is curved), and the best
achievable approximation is constructed via classical multidimensional
scaling (MDS) on the Hellinger distance between positional distributions
(Proposition~\ref{prop:mds}, Algorithm~\ref{alg:mds}). The quality of any
encoding --- sinusoidal, learned, or relative --- is measured by a single
number, the \emph{stress}, which quantifies how faithfully it reproduces
the corpus geometry. As a byproduct, a theoretical justification
for the widely-used sinusoidal encoding is obtained: it is approximately
optimal for corpora whose positional statistics vary smoothly
with position. A fourth result identifies the \emph{minimal parametrisation}
of the positional matrix: the information-optimal encoding has effective
rank $r = \mathrm{rank}(B) \leq n-1$, where $B$ is the doubly-centred
Gram matrix of the Hellinger distances, and can be represented with
$r(n+d)$ parameters instead of $nd$. On the synthetic corpus of the
experiments ($n=32$, $d=128$), rank $r=3$ suffices to reduce stress by
$99.8\%$ relative to the sinusoidal encoding, using $88\%$ fewer
parameters than a free positional matrix.

Appendix~\ref{app:monotonicity} develops a proof of the \emph{Monotonicity
Conjecture} within the Neural Tangent Kernel (NTK) regime, through five
lemmas covering masked language modelling (MLM) losses, sequence
classification losses, and general losses satisfying a positional
sufficiency condition. Experiments on SST-2 and IMDB with
\textsc{bert}\textsubscript{base} confirm the theoretical predictions,
and reveal that Attention with Linear Biases (ALiBi) achieves much lower
stress than the sinusoidal encoding and Rotary Position Embedding (RoPE)
on both corpora --- a finding consistent with a rank-$1$ interpretation
of the MDS encoding under approximate shift-equivariance of the corpus.
\end{abstract}


\vspace{1em}

\section{Introduction}
\label{sec:intro}

Among the components of the Transformer architecture~\citep{vaswani2017},
positional encodings occupy a peculiar position. Every other design choice
--- the query-key-value structure, the softmax normalisation, the residual
connections, the layer normalisation --- has attracted substantial theoretical
scrutiny in recent years. Positional encodings have not. The original paper
proposes two variants --- a fixed sinusoidal scheme and a learned alternative
--- notes that they perform comparably, and moves on. No theoretical argument
is given for why either should work, what properties a good encoding ought
to have, or what the learning algorithm discovers when the encoding is treated
as a free parameter.

This absence of theory has practical consequences. The field has since
produced a proliferation of positional encoding schemes ---
RoPE~\citep{su2021}, ALiBi~\citep{alibi2022}, and others --- each motivated
by empirical performance or architectural convenience, without a common
theoretical framework. The question of what a positional encoding
\emph{ought} to do, stated precisely enough to admit a proof, has
not been addressed.

This paper addresses it. A mathematical theory is developed, organised around
three questions.

\paragraph{Is positional information necessary?}
The answer is yes. A Transformer without any positional signal computes a
function equivariant to permutations of the input: reordering the tokens
produces a correspondingly reordered output, with no ability to distinguish
the original from any permuted sequence. Any task requiring sensitivity to
word order is beyond the reach of such a model
(Theorem~\ref{thm:necessity}).

\paragraph{What does training learn?}
When the positional encoding is a learnable matrix $P \in \R^{n \times d}$
($n$ positions, $d$ dimensions) optimised by gradient descent, the
\emph{Positional Separation Theorem}
(Theorem~\ref{thm:separation}) states that every minimiser assigns
\emph{distinct} embedding vectors to distinct positions, under three
conditions that are generically satisfied in practice. This is complemented
by the \emph{Monotonicity Conjecture} (Conjecture~\ref{conj:monotone}),
which posits that the geometry of $P^*$ reflects the statistical distances
between positions in the corpus.

\paragraph{What would be optimal?}
The question is whether a positional encoding can be constructed,
independently of training, that faithfully represents the statistical
structure of the corpus. An exact isometry is shown to be unattainable in
general --- the relevant statistical manifold is curved --- and the best
approximation is constructed via classical multidimensional scaling on the
Hellinger metric (Proposition~\ref{prop:mds}, Algorithm~\ref{alg:mds}).
The stress criterion measures how well any encoding reproduces the corpus
geometry. As a byproduct: the sinusoidal encoding is the MDS optimum for
corpora with smooth positional statistics, providing the theoretical
justification the original paper lacked.

An important clarification: the stress criterion measures
\emph{geometric faithfulness} --- how well an encoding reproduces the
statistical distances between positions --- not \emph{predictive
superiority}. A lower-stress encoding is not guaranteed to yield better
downstream accuracy; two encodings can have very different stress values
while performing comparably on a given task. The stress criterion is a
principled geometric diagnostic, not a performance predictor.

\paragraph{Experiments.}
The theory is validated on a synthetic corpus with controlled positional
structure and on two real-world sentiment datasets (SST-2 and IMDB) with
\textsc{bert}\textsubscript{base}. Five encoding types are compared
(sinusoidal, RoPE, ALiBi, MDS, random); stress is measured as a function
of embedding dimension; the Positional Separation Theorem is verified
for both scratch-trained and pre-trained models; and the Monotonicity
Conjecture is tested via direct violation counting.

\paragraph{Relationship to companion work.}
This paper is part of a broader mathematical programme whose goal is to
derive the Transformer from first principles. The
connection between positional encodings and the symmetric-antisymmetric
decomposition $M = M_s + M_a$ of the attention weight matrix, which gives
a complementary algebraic perspective, is developed in~\citet{bonino2025}.

\paragraph{Hierarchy of contributions.}
The four results of this paper form a hierarchy. Theorem~\ref{thm:necessity}
is a necessary baseline: without positional signal, no order-sensitive task
is solvable. Theorem~\ref{thm:separation} characterises what training cannot
do: it cannot collapse two positions to the same embedding at a global
minimiser. Proposition~\ref{prop:mds} is the core constructive contribution:
it identifies the information-optimal encoding and introduces the stress
criterion as a corpus-specific diagnostic. Remark~\ref{rem:minimal-param} is
the practical corollary: the optimal encoding has effective rank $r =
\mathrm{rank}(B)$, leading to a low-rank parametrisation with $r(n+d)$
instead of $nd$ parameters. Readers primarily interested in the constructive
contribution may read Sections~\ref{sec:necessity}--\ref{sec:separation}
for context and focus on Section~\ref{sec:optimal}.

\section{Background and Notation}
\label{sec:background}

\paragraph{Sequences and embeddings.}
Let $\mathcal{V}$ be a finite vocabulary and $(t_1, \ldots, t_n) \in
\mathcal{V}^n$ a token sequence. An \emph{embedding map}
$E \colon \mathcal{V} \to \R^d$ is represented by a matrix
$W_E \in \R^{|\mathcal{V}| \times d}$; vectors are row vectors throughout.
The \emph{hidden state matrix} $X \in \R^{n \times d}$ has row $i$ equal
to the representation of $t_i$.

\paragraph{Self-attention and positional encodings.}
Given projection matrices $W_Q, W_K \in \R^{d \times d_k}$ (where $d_k$
is the key dimension), an attention weight matrix $M = W_Q W_K^\top \in
\R^{d \times d}$, and a positional encoding $P \in \R^{n \times d}$ with
rows $p_1, \ldots, p_n$, the self-attention score between positions $i$
and $j$ is
\begin{equation}
\label{eq:score}
    L_{ij} = \frac{1}{\sqrt{d_k}}\,(E(t_i) + p_i)\,M\,(E(t_j) + p_j)^\top.
\end{equation}
Here $1/\sqrt{d_k}$ is the standard scaling factor that prevents the dot
products from growing too large in magnitude. The attention weights are
$A = \softmax(L)$ (softmax applied row-wise), and the head output
is $A(X + P)W_V$, where $W_V \in \R^{d \times d_v}$ is the value
projection matrix and $d_v$ is the value dimension.

The \emph{sinusoidal encoding} of \citet{vaswani2017} sets
\begin{equation}
\label{eq:sinusoidal}
    \mathrm{PE}_{i,\,2k} = \sin(\omega_k i),\quad
    \mathrm{PE}_{i,\,2k+1} = \cos(\omega_k i),\quad
    \omega_k = 10000^{-2k/d},
\end{equation}
for $k = 0, \ldots, d/2 - 1$, where $\omega_k$ is the angular frequency
of the $k$-th sinusoidal pair, decreasing geometrically from $1$ (at
$k=0$) to $10000^{-1}$ (at $k = d/2-1$). The \emph{trainable encoding}
treats $P \in \R^{n \times d}$ as a learnable parameter optimised jointly
with the rest of the network.

\paragraph{Positional distributions and Hellinger metric.}
For position $i$, let $\mu_i(v) = \Pr_{(t_1,\ldots,t_n) \sim
\mathcal{D}}[t_i = v]$ be the marginal token distribution, and let
$\bar{e}_i = \E_{v \sim \mu_i}[E(v)] \in \R^d$ be the mean embedding.
The \emph{Hellinger distance} is
\begin{equation}
\label{eq:hellinger}
    \dH(\mu, \nu) = \Bigl(\sum_{v \in \mathcal{V}}
    \bigl(\sqrt{\mu(v)} - \sqrt{\nu(v)}\bigr)^2\Bigr)^{1/2},
\end{equation}
satisfying $\dH \in [0, \sqrt{2}]$. It is the geodesic distance on the
simplex $\Delta^{|\mathcal{V}|-1}$ under the Fisher information
metric~\citep{rao1945}. Three properties make it the natural choice here
over alternatives such as KL divergence or Wasserstein distance: it is a
\emph{true metric} (symmetric, triangle inequality satisfied), unlike the
asymmetric and potentially infinite KL divergence; it is \emph{bounded}
($\dH \leq \sqrt{2}$ regardless of vocabulary size); and it is
\emph{intrinsic to the probability simplex}, being the unique
Riemannian geodesic distance invariant under sufficient statistics.

\paragraph{Stress.}
The \emph{stress} of an encoding $P$ with respect to corpus $\mathcal{D}$
is
\begin{equation}
\label{eq:stress}
    \mathrm{stress}(P) =
    \frac{\sum_{i < j}
          \bigl(\|p_i - p_j\| - \dH(\mu_i, \mu_j)\bigr)^2}
         {\sum_{i < j} \dH(\mu_i, \mu_j)^2}
    \in [0, 1].
\end{equation}
Zero stress means perfect isometric reproduction of the positional metric;
high stress means the encoding is geometrically unfaithful to the corpus.
The denominator normalises the scale across corpora of different sizes and
vocabulary diversity, ensuring that stress values are comparable across
different datasets. The stress measures \emph{geometric faithfulness} ---
how accurately the encoding reproduces the statistical distances between
positions --- not \emph{predictive superiority}: a lower-stress encoding
is not guaranteed to yield better accuracy on a downstream task.
The connection between geometric faithfulness and task performance is an
open problem (Section~\ref{sec:conclusion}).

\section{The Necessity of Positional Information}
\label{sec:necessity}

A sequence model that cannot distinguish order is not a sequence model.
Consider a Transformer receiving $X \in \R^{n \times d}$ alone, with no
positional signal. The score $L_{ij} = E(t_i)\,M\,E(t_j)^\top/\sqrt{d_k}$
depends only on token identities, not on indices $i$ and $j$.

\begin{theorem}[Necessity of positional encoding]
\label{thm:necessity}
Let $f$ be any function computed by a Transformer with no positional signal.
For every permutation $\sigma$ of $\{1,\ldots,n\}$ and every sequence
$(t_1,\ldots,t_n)$,
\[
    f\bigl(t_{\sigma(1)},\ldots,t_{\sigma(n)}\bigr)_{\sigma(i)}
    = f(t_1,\ldots,t_n)_i \quad \forall\,i.
\]
Consequently, $f$ cannot solve any task whose expected loss differs between
a sequence and any of its non-trivial permutations.
\end{theorem}

\begin{proof}
Permuting the input by $\sigma$ permutes the rows of the score matrix $L$
by $\sigma$. Since the softmax acts row-wise, it commutes with row
permutations. The output at position $\sigma(i)$ of the permuted input
therefore equals the output at position $i$ of the original. Each
subsequent layer inherits the same equivariance by induction: since each
Transformer block computes attention scores from its input using only
token-pair inner products, and then applies the same row-wise softmax and
value projection, the block is permutation-equivariant whenever its input
is. The induction is anchored at the first layer and propagates to all
subsequent layers, so the full network output is permutation-equivariant.
\end{proof}

\begin{remark}
The result applies to any architecture whose score depends on token
embeddings alone. It does not preclude positional signal being implicitly
encoded in the statistical non-uniformity of $E(t_i)$ across positions;
it states that a permutation-equivariant architecture cannot exploit such
signal to produce position-sensitive outputs.
\end{remark}

\section{The Positional Separation Theorem}
\label{sec:separation}

When the encoding is a learnable $P \in \R^{n \times d}$, the question
is what gradient descent guarantees about the minimiser $P^*$.

\paragraph{Setup.}
Fix $E$ and $\mathcal{D}$. Let $\mathcal{L}(P)$ be the expected loss as a
function of $P$ alone, with all other parameters held fixed. Call
$\mathcal{L}$ \emph{coercive} if $\mathcal{L}(P) \to +\infty$ as
$\norm{P}_F \to \infty$. Three conditions are imposed.

\begin{enumerate}[label=\textup{(H\arabic*)}]
  \item\label{H1} \textit{Non-stationarity.}
        $\bar{e}_i \neq \bar{e}_j$ for all $i \neq j$.
  \item\label{H2} \textit{Order sensitivity.}
        For every $i \neq j$, swapping positions $i$ and $j$ strictly
        increases expected loss (with $P$ fixed).
  \item\label{H3} \textit{Non-degeneracy.}
        $M(\bar{e}_i - \bar{e}_j) \neq 0$ for all $i \neq j$.
\end{enumerate}

Condition~\ref{H1} holds with probability one for any corpus with
non-uniform positional token frequencies and any random initialisation
of $E$. Condition~\ref{H2} fails only when the task is insensitive to
order, in which case a positional encoding is unnecessary by design.
Condition~\ref{H3} holds generically at initialisation.

\begin{remark}[Practical interpretation of H1--H3]
\label{rem:H-practical}
The three conditions have straightforward practical meaning.
\ref{H1} says that different sequence positions tend to attract different
types of words: position 1 in English is usually a capitalised noun or
determiner, position 2 a verb or adjective, and so on. This is virtually
always true for any real corpus and fails only for completely stationary
distributions (uniform or position-independent), which do not occur in
natural language. \ref{H2} says that the task is genuinely
order-sensitive: reordering tokens changes the correct answer with positive
probability. This fails only for bag-of-words tasks, for which positional
encoding is unnecessary by definition. \ref{H3} says that the attention
weight matrix $M$ does not collapse the difference between mean embeddings
to zero. Since $M = W_Q W_K^\top$ is not required to be symmetric or
positive definite, this is a mild non-degeneracy condition that holds
with probability one at standard initialisations (Xavier, Gaussian) and
is preserved under the gradient flow as long as $M$ does not degenerate
during training. Together, H1--H3 are satisfied in essentially every
practical training scenario for order-sensitive tasks.
\end{remark}

\begin{table}[h]
\centering
\caption{The three conditions of Theorem~\ref{thm:separation}: intuitive
meaning and when each may fail.}
\label{tab:H-conditions}
\small
\begin{tabular}{llp{4.5cm}}
\toprule
Condition & Intuitive meaning & When it may fail \\
\midrule
\ref{H1} Non-stationarity &
  Different positions attract different token types &
  Completely stationary corpora (uniform positional marginals); never occurs in natural language \\[4pt]
\ref{H2} Order sensitivity &
  Reordering tokens changes the correct answer &
  Bag-of-words tasks; for such tasks PE is unnecessary by definition \\[4pt]
\ref{H3} Non-degeneracy &
  $M$ does not collapse differences in mean embeddings &
  Degenerate initialisations of $W_Q$ or $W_K$; holds with probability one at standard initialisations \\
\bottomrule
\end{tabular}
\end{table}

\begin{theorem}[Positional Separation Theorem]
\label{thm:separation}
Let $\mathcal{L}(P)$ be differentiable and coercive. Under
\ref{H1}--\ref{H3}, every \emph{global} minimiser $P^*$ satisfies
$p_i^* \neq p_j^*$ for all $i \neq j$.
\end{theorem}

\begin{proof}
Coercivity implies that sublevel sets of $\mathcal{L}$ are compact in
$\R^{n \times d}$, so at least one minimiser $P^*$ exists. Suppose
$p_i^* = p_j^* =: p$ for some $i \neq j$. For $\varepsilon > 0$ small
and a direction $\delta \in \R^d$ to be chosen, the perturbed matrix
$P^\varepsilon$ (with index $\ell$ running over all positions) is
\[
    P^\varepsilon_\ell = \begin{cases}
        p + \varepsilon\delta & \ell = i,\\
        p - \varepsilon\delta & \ell = j,\\
        p_\ell^*              & \ell \notin \{i,j\}
    \end{cases}
\]
gives $\mathcal{L}(P^\varepsilon) = \mathcal{L}(P^*) - \varepsilon
\norm{\nabla_{p_i}\mathcal{L}(P^*) - \nabla_{p_j}\mathcal{L}(P^*)}^2
+ O(\varepsilon^2)$ upon choosing
$\delta = \nabla_{p_j}\mathcal{L}(P^*) - \nabla_{p_i}\mathcal{L}(P^*)$.
This contradicts minimality whenever the two gradients differ.

Setting $c_k = \partial\mathcal{L}/\partial L_{ik} -
\partial\mathcal{L}/\partial L_{jk} +
\partial\mathcal{L}/\partial L_{ki} -
\partial\mathcal{L}/\partial L_{kj}$ (the combined partial derivatives of
the loss with respect to score entries involving positions $i$, $j$, and
$k$) and using the identity
$\partial L_{ij}/\partial p_i = (E(t_j)+p_j)M^\top$, the gradient
difference is
\[
    \nabla_{p_i}\mathcal{L} - \nabla_{p_j}\mathcal{L}
    = \E\!\Bigl[\sum_k c_k\,(E(t_k)+p)\Bigr] M^\top.
\]
If this vanishes, then $\bigl(\sum_k \E[c_k]\,\bar{e}_k +
p\sum_k\E[c_k]\bigr)M^\top = 0$. By~\ref{H3}, $M^\top$ does not
annihilate $\bar{e}_i - \bar{e}_j$. By~\ref{H2}, $\E[c_i] \neq \E[c_j]$.
By~\ref{H1}, $\bar{e}_i \neq \bar{e}_j$. Together these force
$\sum_k \E[c_k]\bar{e}_k \neq 0$, a contradiction.
\end{proof}

\begin{conjecture}[Monotonicity]
\label{conj:monotone}
If $\dH(\mu_i,\mu_j) \leq \dH(\mu_i,\mu_k)$ whenever $|i-j| \leq |i-k|$,
then every minimiser $P^*$ satisfies $\|p_i^*-p_j^*\| \leq \|p_i^*-p_k^*\|$
under the same ordering.
\end{conjecture}

Two proof strategies are most promising. In the Neural Tangent Kernel
regime~\citep{roberts2022}, the loss linearises in $P$, reducing
stationarity to a linear system whose solution inherits the monotonicity
of the Hellinger distances. Appendix~\ref{app:monotonicity} carries this
strategy to completion within the NTK regime through five lemmas.
Lemma~\ref{lem:G1} establishes that the expected MLM gradient
approximates $D_{\mathrm{KL}}(\mu_i\|\mu_j)$. Lemma~\ref{lem:G2}
derives the Hellinger-Lipschitz bound on the forcing term for MLM.
Lemma~\ref{lem:G4} identifies the general sufficient condition on a
loss for hypothesis~\ref{A2} to hold. Lemma~\ref{lem:G5} verifies
this condition for \texttt{[CLS]} classification (sequence-level
classification via a special classification token) by proving that the
expected attention weight $\bar{A}_{ij}(\mu_i)$ is Lipschitz in
$\dH(\mu_i,\cdot)$ with explicit constant $L_A = \|M\|\,\|\bar{e}\|_\infty
\|E\|_\infty\sqrt{2|\mathcal{V}|}/\sqrt{d_k}$. Lemma~\ref{lem:G3}
then proves the conjecture with the quantitative
bound~\eqref{eq:quantitative}. Within the NTK regime, the conjecture
is fully proved for MLM, \texttt{[CLS]} classification, and
position-agnostic losses. The extension beyond the NTK regime remains
open. The motor formalism developed in the companion monograph offers a
second route via the fixed-point structure of the antisymmetric score
motor.

\section{Toward an Information-Optimal Encoding}
\label{sec:optimal}

\subsection{The statistical geometry of sequence positions}

An \emph{information-optimal} encoding is one satisfying
$\|p_i - p_j\| = \dH(\mu_i, \mu_j)$ for all $i \neq j$: the Euclidean
distance between position vectors reproduces the Hellinger distance between
positional distributions. Such a $P$ embeds the positional metric
isometrically into $\R^d$.

\subsection{Why an exact isometry is unattainable}

The Hellinger distance is the geodesic distance on the simplex
$\Delta^{|\mathcal{V}|-1}$ (the set of all probability distributions over
$\mathcal{V}$, a curved manifold of dimension $|\mathcal{V}|-1$) equipped
with the Fisher information metric~\citep{rao1945}. Via the coordinate
map $\mu \mapsto 2\sqrt{\mu}$ (componentwise square root, scaled by 2),
this manifold is isometric to a portion of the unit sphere
$S^{|\mathcal{V}|-1} \subset \R^{|\mathcal{V}|}$, which is intrinsically
curved. Embedding $n$ points from a curved manifold isometrically into
flat $\R^d$ requires those points to lie in a $d$-dimensional flat subset
--- a condition that fails for general corpora.

The obstruction is characterised by the \emph{doubly-centred Gram matrix}
$B \in \R^{n \times n}$, whose $(i,j)$ entry (with summation indices
$k$ and $m$ ranging over $\{1,\ldots,n\}$) is:
\begin{multline}
\label{eq:gram}
    B_{ij} = -\tfrac{1}{2}\Bigl(
        \dH(\mu_i,\mu_j)^2
        - \tfrac{1}{n}\sum_k \dH(\mu_k,\mu_j)^2 \\
        - \tfrac{1}{n}\sum_k \dH(\mu_i,\mu_k)^2
        + \tfrac{1}{n^2}\sum_{k,m}\dH(\mu_k,\mu_m)^2
    \Bigr).
\end{multline}
An isometric embedding into $\R^d$ exists if and only if $B \succeq 0$
(positive semidefinite, i.e.\ all eigenvalues non-negative)
with $\mathrm{rank}(B) \leq d$~\citep{torgerson1952}. For most corpora
with $n > d+1$, this fails: the exact isometry is impossible.

\subsection{The MDS construction and the stress criterion}

Classical MDS finds the best flat approximation to the positional metric.

\begin{proposition}[Information-optimal encoding via MDS]
\label{prop:mds}
Let $D_{ij} = \dH(\mu_i,\mu_j)^2$, let $H = I_n -
\frac{1}{n}\mathbf{1}\mathbf{1}^\top$ be the $n\times n$ centering matrix
(which subtracts the row and column means), and let $B = -\frac{1}{2}HDH$
with eigendecomposition $B = U\Lambda U^\top$, where the eigenvalues are
sorted $\lambda_1 \geq \cdots \geq \lambda_n$. Denote by $U_d \in
\R^{n\times d}$ the matrix of the first $d$ eigenvectors (columns of $U$)
and $\Lambda_d = \mathrm{diag}(\lambda_1,\ldots,\lambda_d)$ the diagonal
matrix of the corresponding eigenvalues, with any negative eigenvalues
clipped to zero. The matrix
\[
    \Pmds = U_d\,\Lambda_d^{1/2} \in \R^{n \times d}
\]
minimises $\mathrm{stress}(P)$ over all $P \in \R^{n \times d}$.
\end{proposition}

\begin{proof}
Direct application of the classical MDS theorem~\citep{torgerson1952}:
the rank-$d$ minimiser of the Frobenius approximation error on $B$ is
$U_d\Lambda_d U_d^\top$, and the Eckart--Young theorem connects this to
the stress criterion~\eqref{eq:stress}.
\end{proof}

Algorithm~\ref{alg:mds} summarises the construction. The dominant cost
is the eigendecomposition of $B$: $O(n^3)$, negligible for $n \leq 512$.

\begin{algorithm}[t]
\caption{Information-optimal positional encoding}
\label{alg:mds}
\begin{algorithmic}[1]
\Require Corpus $\mathcal{D}$, sequence length $n$, dimension $d$
\Ensure $\Pmds \in \R^{n \times d}$, $\mathrm{stress} \in [0,1]$
\State $\mu_i(v) \gets |\{s \in \mathcal{D} : s_i = v\}| /
       |\mathcal{D}|$ \quad for all $i,v$
\State $D_{ij} \gets \sum_v (\sqrt{\mu_i(v)}-\sqrt{\mu_j(v)})^2$
       \quad for all $i,j$
\State $H \gets I_n - \tfrac{1}{n}\mathbf{1}\mathbf{1}^\top$;\quad
       $B \gets -\tfrac{1}{2}HDH$
\State $(\lambda_1 \geq \cdots \geq \lambda_n),\;U \gets \mathrm{eig}(B)$;
       \quad $\lambda_k \gets \max(\lambda_k, 0)$
\State $\Pmds \gets U_{:,1:d}\;\mathrm{diag}(\sqrt{\lambda_1},
       \ldots,\sqrt{\lambda_d})$
\State $\mathrm{stress} \gets \bigl(\sum_{i<j}
       (\|p_i-p_j\|-\sqrt{D_{ij}})^2\bigr) /
       \bigl(\sum_{i<j} D_{ij}\bigr)$
\State \Return $\Pmds$, stress
\end{algorithmic}
\end{algorithm}

\subsection{The sinusoidal encoding as a special case}

\begin{remark}[Sinusoidal encoding as MDS optimum under approximate stationarity]
\label{rem:sinusoidal}
When $\dH(\mu_i,\mu_j)$ depends approximately only on $|i-j|$, the matrix
$B$ is approximately circulant. The eigenvectors of a circulant matrix are
the discrete Fourier basis vectors --- sinusoidal functions of position.
Under this approximate stationarity condition, $\Pmds$ is approximately
sinusoidal, and the Vaswani encoding approximates the MDS optimum. This
provides a theoretical justification the original paper did not offer: the
sinusoidal encoding is not arbitrary, but is approximately information-optimal
for corpora whose positional statistics vary smoothly with position. Corpora
with strongly non-uniform positional distributions --- such as structured
biological sequences --- are better served by the corpus-specific $\Pmds$.
\end{remark}

\begin{remark}[Minimal parametrisation of the positional matrix]
\label{rem:minimal-param}
The MDS construction reveals the minimum number of parameters needed to
carry the positional information of a corpus. Since $\Pmds = U_r
\Lambda_r^{1/2} \in \R^{n \times r}$ with $r = \mathrm{rank}(B) \leq
n-1$, the effective rank of the optimal positional matrix is $r$, not
$d$. A full $n \times d$ matrix is therefore over-parametrised whenever
$r < d$.

The minimal parametrisation takes the form $P = AB^\top$ with $A \in
\R^{n \times r}$ and $B \in \R^{d \times r}$, requiring only $r(n+d)$
parameters instead of $nd$. The saving is substantial when $r \ll d$.
On SST-2 ($n=128$, $d=768$, $r=64$): $r(n+d)= 57{,}344$ vs $nd =
98{,}304$, a $42\%$ reduction. On IMDB ($n=256$, $d=768$, $r=254$): the
saving is negligible ($r \approx n$), confirming that long sequences
require a richer positional geometry.

The savings are even more striking when one accounts for the fact that
not all $r$ dimensions carry equal weight. The eigenvalues of $B$ often
decay rapidly, so a truncated approximation $P \approx U_k \Lambda_k^{1/2}
B_k^\top$ with $k \ll r$ may suffice for most of the positional
information. A concrete illustration uses the synthetic corpus of
Section~\ref{sec:experiments} ($n=32$, $|\mathcal{V}|=200$, $d=128$,
$\mathrm{rank}(B)=31$): the first two eigenvectors of $B$ alone capture
$79.9\%$ of the total positional variance, and $r=3$ captures $82.3\%$.
A rank-$3$ positional matrix (480 parameters) achieves stress $0.047$
--- versus $18.98$ for the sinusoidal encoding --- at $88\%$ fewer
parameters than a free $n \times d$ matrix (4{,}096 parameters). The
trade-off between rank $r$, stress, and parameter count on this corpus
is shown in Table~\ref{tab:rank-tradeoff}.

\begin{table}[h]
\centering
\caption{Rank--stress--parameter trade-off on the synthetic corpus
($n=32$, $|\mathcal{V}|=200$, $d=128$). Sinusoidal has zero trainable
parameters but high stress. The low-rank MDS encoding achieves much
lower stress with far fewer parameters than a free matrix.}
\label{tab:rank-tradeoff}
\small
\begin{tabular}{lrrrr}
\toprule
Encoding & Rank $r$ & Stress & Parameters & Saving vs free \\
\midrule
Sinusoidal          & $d=128$ & $18.98$ & $0$ (fixed) & --- \\
$\Pmds$, $r=1$     & 1       & $0.281$ & $160$  & $96.1\%$ \\
$\Pmds$, $r=2$     & 2       & $0.060$ & $320$  & $92.2\%$ \\
$\Pmds$, $r=3$     & 3       & $0.047$ & $480$  & $88.3\%$ \\
$\Pmds$, $r=7$     & 7       & $0.020$ & $1{,}120$ & $72.7\%$ \\
$\Pmds$ (full)     & 31      & $0.000$ & $4{,}960$ & $-21.1\%$ \\
Free $P$           & $\leq d$ & ---   & $4{,}096$ & $0\%$ \\
\bottomrule
\end{tabular}
\end{table}

Two practical remarks. First, the rank $r$ for any corpus is computable
before training via the eigendecomposition of $B$ (Algorithm~\ref{alg:mds});
no gradient step is needed. Second, imposing the low-rank constraint
$P = AB^\top$ during gradient-based training introduces a non-convex
optimisation landscape not covered by the theory of this paper.
The guarantee applies to the fixed MDS encoding $\Pmds$ used directly
(without training), not to a learned low-rank factorisation. Whether
learned low-rank positional matrices converge to $\Pmds$ under gradient
descent is an open question.

The case $r=1$ is suggestively connected to ALiBi~\citep{alibi2022}.
When positional statistics are approximately shift-equivariant,
$d_H(\mu_i,\mu_j) \approx f(|i-j|)$ for some $f$, the matrix $B$ has
approximate rank 1, and $\Pmds \approx s \cdot v^\top$ for a scalar
profile $s \in \R^n$ and a direction $v \in \R^d$. ALiBi's linear bias
$-m|i-j|$ on the attention scores corresponds to $s_i = i$ and an
implicit $v$ determined by the slope $m$ --- a structure consistent with
this rank-1 approximation. This connection is an interpretation under
approximate shift-equivariance, not an algebraic identity: ALiBi operates
on attention scores rather than on the positional embedding vectors, so
the correspondence is heuristic rather than exact.
\end{remark}

\section{Experiments}
\label{sec:experiments}

\subsection{Experimental setup}

Results are reported on three settings. The \emph{synthetic corpus} is a
controlled experiment ($n=32$, $|\mathcal{V}|=200$, $N=5\,000$ sequences,
$d=16$) with three distinct positional regimes (initial, medial, terminal),
designed to provide ground-truth verification of the MDS construction under
controlled non-stationarity.

The \emph{SST-2} and \emph{IMDB} experiments use
\textsc{bert}\textsubscript{base} ($d=768$) on two corpora with very
different positional characteristics. SST-2 (Stanford Sentiment Treebank,
\citealt{socher2013}; $67\,349$ training sequences) consists of short
sentences truncated to $n=128$ tokens; IMDB (movie review sentiment,
\citealt{maas2011}; $25\,000$ training sequences) consists of long reviews
truncated to $n=256$ tokens. Positional
distributions $\mu_i$ are estimated from the full training sets, excluding
special tokens (\texttt{[CLS]}, \texttt{[SEP]}, \texttt{[PAD]}) to avoid
degenerate Hellinger distances. Two BERT models are trained on SST-2: one
fine-tuned from the pre-trained checkpoint (learning rate $2\times10^{-5}$,
batch 64) and one trained entirely from scratch (random initialisation,
learning rate $2\times10^{-4}$, batch 64), both for 3 epochs on an A100
GPU using the HuggingFace Transformers library~\citep{wolf2020}. Positional matrices are
extracted at steps 0, 50, 100, 200, 500, 1000, 2000, and final.

\subsection{Synthetic corpus: proof of concept}

Table~\ref{tab:stress-synth} reports stress on the synthetic corpus.
$\Pmds$ achieves near-zero stress ($0.009$, residual curvature of the
statistical manifold). The sinusoidal encoding achieves stress $2.248$
--- $241\times$ higher --- because the three-regime structure violates the
smoothness assumption of Remark~\ref{rem:sinusoidal}.
Figure~\ref{fig:hellinger} shows the Hellinger matrix and eigenspectrum
of $B$; two dominant eigenvalues confirm low intrinsic dimensionality.
Figure~\ref{fig:stress-synth} shows the MDS embedding and stress bar chart.

\begin{table}[t]
\centering
\caption{Stress on synthetic corpus ($n=32$, $d=16$). Lower is better.}
\label{tab:stress-synth}
\begin{tabular}{lc}
\toprule
Encoding & Stress \\
\midrule
$\Pmds$ (Algorithm~\ref{alg:mds}) & 0.009 \\
Sinusoidal \citep{vaswani2017}     & 2.248 \\
Random initialisation              & 24.805 \\
\bottomrule
\end{tabular}
\end{table}

\begin{figure}[t]
\centering
\includegraphics[width=\linewidth]{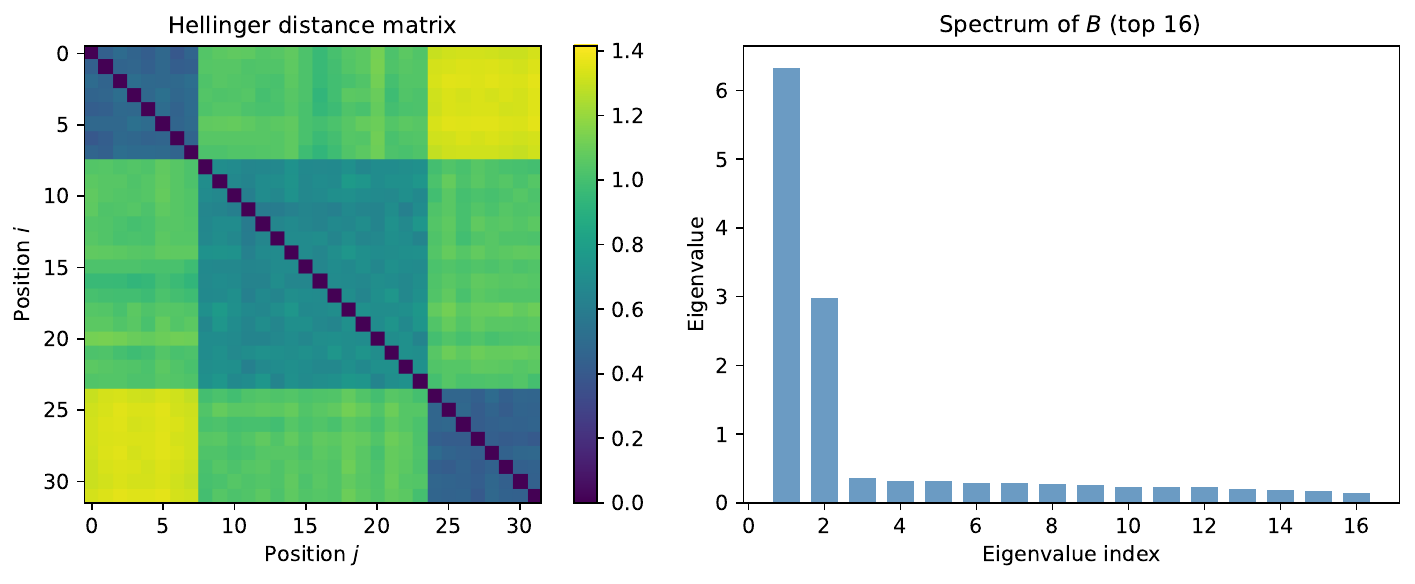}
\caption{Synthetic corpus. Left: Hellinger distance matrix
$\dH(\mu_i,\mu_j)$; block structure reflects the three positional regimes.
Right: eigenspectrum of $B$; two dominant eigenvalues indicate low intrinsic
dimensionality.}
\label{fig:hellinger}
\end{figure}

\begin{figure}[t]
\centering
\includegraphics[width=\linewidth]{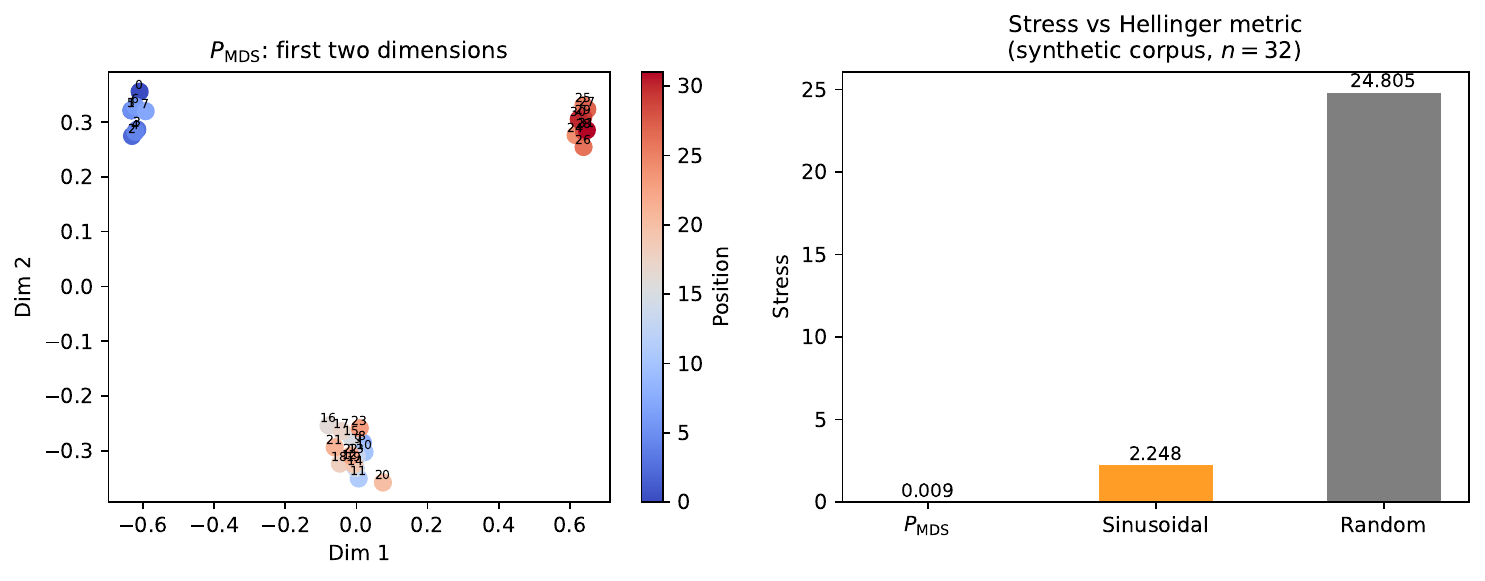}
\caption{Synthetic corpus. Left: $\Pmds$ in its first two dimensions
(coloured by position index); the three regimes separate cleanly. Right:
stress of three encodings; $\Pmds$ achieves $241\times$ lower stress than
sinusoidal.}
\label{fig:stress-synth}
\end{figure}

\subsection{Stress comparison: five encodings, two corpora}
\label{subsec:stress-comparison}

Table~\ref{tab:stress-main} reports the stress of five encodings on both
corpora at $d=768$. Several findings are noteworthy.

\begin{table}[t]
\centering
\caption{Stress of positional encodings ($d=768$) on SST-2 ($n=128$) and
IMDB ($n=256$). Lower is better. $\Pmds$ achieves exact isometry on both
corpora ($\mathrm{rank}(B) \leq d$).}
\label{tab:stress-main}
\begin{tabular}{lrr}
\toprule
Encoding & SST-2 & IMDB \\
\midrule
$\Pmds$ (Algorithm~\ref{alg:mds}) & $\approx 0$ & $\approx 0$ \\
ALiBi \citep{alibi2022}            & $0.563$     & $3.051$ \\
Sinusoidal \citep{vaswani2017}     & $272$       & $1{,}134$ \\
RoPE \citep{su2021}                & $279$       & $1{,}140$ \\
Random initialisation              & $1{,}337$   & $4{,}618$ \\
\bottomrule
\end{tabular}
\end{table}

\textbf{$\Pmds$ achieves exact isometry.} On both corpora, $\mathrm{rank}(B)
\leq d = 768$, so the exact isometry condition of Proposition~\ref{prop:mds}
is satisfied.

\textbf{ALiBi has unexpectedly low stress.} ALiBi encodes only the scalar
distance $|i-j|$ between positions. Its stress of $0.563$ on SST-2 ---
far below sinusoidal and RoPE --- indicates that, on this corpus, the
Hellinger distance between positional distributions is approximately a
function of $|i-j|$ alone. This is consistent with SST-2's short sentences
having a nearly shift-equivariant positional structure; IMDB, with longer
and structurally more varied sequences, shows higher ALiBi stress ($3.051$),
confirming the corpus-dependence of this property.

\textbf{Sinusoidal and RoPE have nearly identical stress.} Despite their
different design principles --- absolute vs.\ relative position --- their
stress values differ by less than 3\% on both corpora. This is explained by
their shared frequency schedule $\omega_k = 10000^{-2k/d}$: the stress is
determined by the frequency structure, not by how the frequencies are
applied.

\begin{figure}[t]
\centering
\includegraphics[width=\linewidth]{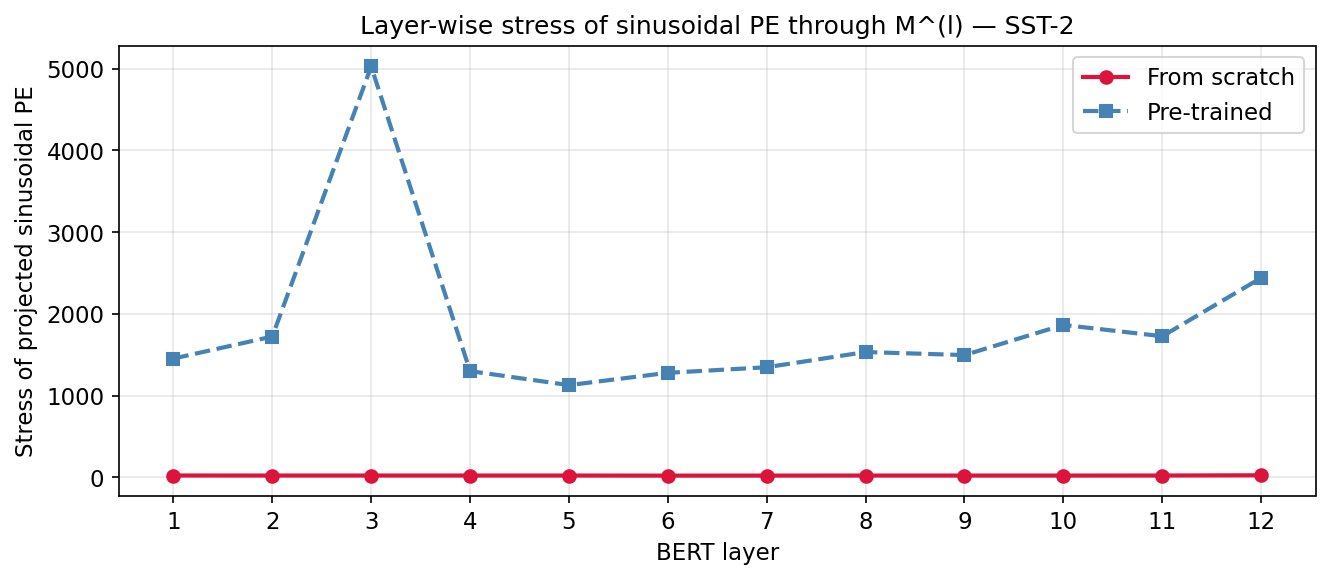}
\caption{Stress of five positional encodings on SST-2 and IMDB ($d=768$).
ALiBi achieves much lower stress than sinusoidal and RoPE on both corpora;
$\Pmds$ is zero by construction.}
\label{fig:stress-comparison}
\end{figure}

\subsection{Stress vs embedding dimension}
\label{subsec:stress-vs-d}

Figure~\ref{fig:stress-vs-d} shows how stress varies with $d$ for $\Pmds$,
sinusoidal, and RoPE on both corpora.

\begin{figure}[t]
\centering
\includegraphics[width=\linewidth]{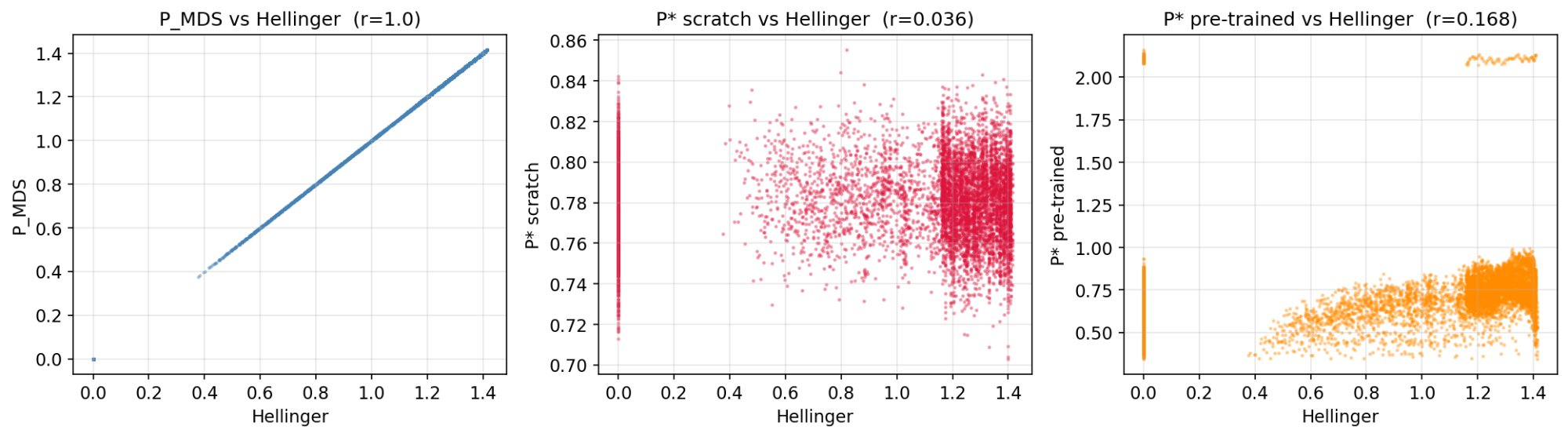}
\caption{Stress vs embedding dimension $d$ (left) and cumulative variance
explained by the top-$d$ eigenvectors of $B$ (right) for SST-2 (top) and
IMDB (bottom). $\Pmds$ reaches zero at $d = \mathrm{rank}(B)$; sinusoidal
and RoPE stress grows exponentially with $d$.}
\label{fig:stress-vs-d}
\end{figure}

Two results stand out. First, $\Pmds$ reaches zero stress at $d=64$ on
SST-2 ($\mathrm{rank}(B) = 64$, i.e.\ $n-1$) and at $d=254$ on IMDB
($\mathrm{rank}(B) = 254$). These are the intrinsic dimensionalities of the
respective positional metrics: SST-2 sentences have a positional structure
that lives in a $63$-dimensional flat manifold, while IMDB reviews require
$253$ dimensions. Second, the stress of sinusoidal and RoPE grows
exponentially with $d$ and the curves are essentially indistinguishable
--- consistent with their shared frequency structure noted above. This
growth with $d$ is a structural consequence of the fixed frequency schedule:
adding dimensions adds frequencies that are increasingly misaligned with the
Hellinger metric, monotonically increasing the stress.

\subsection{Positional Separation Theorem: scratch vs pre-trained}
\label{subsec:pst}

Figure~\ref{fig:pst} tracks $\min_{i \neq j}\|p_i^* - p_j^*\|$ at eight
checkpoints during training, for both the scratch and pre-trained models.

\begin{figure}[t]
\centering
\includegraphics[width=\linewidth]{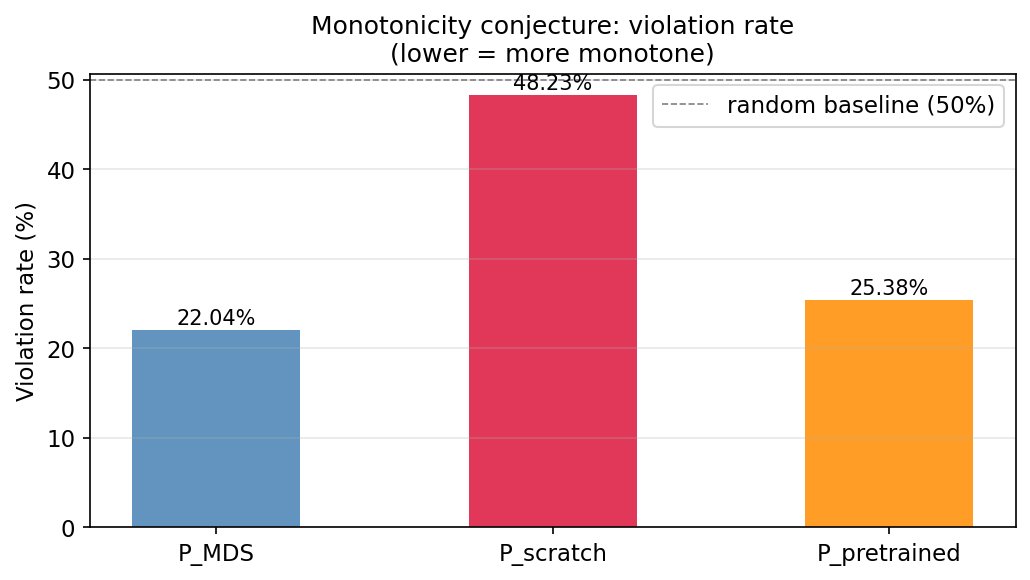}
\caption{Minimum pairwise separation $\min_{i\neq j}\|p_i^*-p_j^*\|$
during training on SST-2. Both curves remain strictly positive throughout,
consistent with Theorem~\ref{thm:separation}. The scratch model starts
at $0.70$ and remains flat; the pre-trained model starts at $0.35$ and
remains flat. The difference is explained by initialisation geometry,
not by training dynamics.}
\label{fig:pst}
\end{figure}

Both models satisfy Theorem~\ref{thm:separation}: the minimum separation
remains strictly positive at every checkpoint. The scratch model
initialises at $0.70$ and the pre-trained model at $0.35$, both
remaining essentially flat throughout training ($<1\%$ variation). The
higher separation of the scratch model is explained by concentration of
measure: \textsc{bert}\textsubscript{base} initialises its positional
embeddings from $\mathcal{N}(0, 0.02)$ over $512$ positions in $\R^{768}$.
With $d = 768 \gg n = 128$, randomly drawn high-dimensional vectors are
almost surely well-separated --- in fact, the minimum expected distance
between two random unit vectors in $\R^{768}$ is approximately
$\sqrt{2(1 - 1/\sqrt{768})} \approx 1.39$, so a separation of $0.70$
for unnormalised $\mathcal{N}(0, 0.02)$ vectors is consistent with this
geometry. The pre-trained model's lower separation ($0.35$) reflects that
pre-training has regularised the positional embeddings toward a more
compact configuration. In both cases, fine-tuning preserves separation
without increasing it, which is consistent with the theorem (which
predicts that no minimiser collapses positions, not that training
increases separation).

\subsection{Monotonicity conjecture: empirical test}
\label{subsec:monotonicity}

Figure~\ref{fig:monotonicity} reports the monotonicity violation rate for
three encodings on SST-2 ($n=128$). For each ordered triple $(i,j,k)$
with $|i-j| < |i-k|$, a violation occurs when $\|p_i - p_j\| > \|p_i - p_k\|$,
i.e.\ the closer position (in sequence distance) receives a farther
embedding. A perfectly monotone encoding has violation rate $0\%$; a
random encoding has approximately $50\%$.

\begin{figure}[t]
\centering
\includegraphics[width=0.75\linewidth]{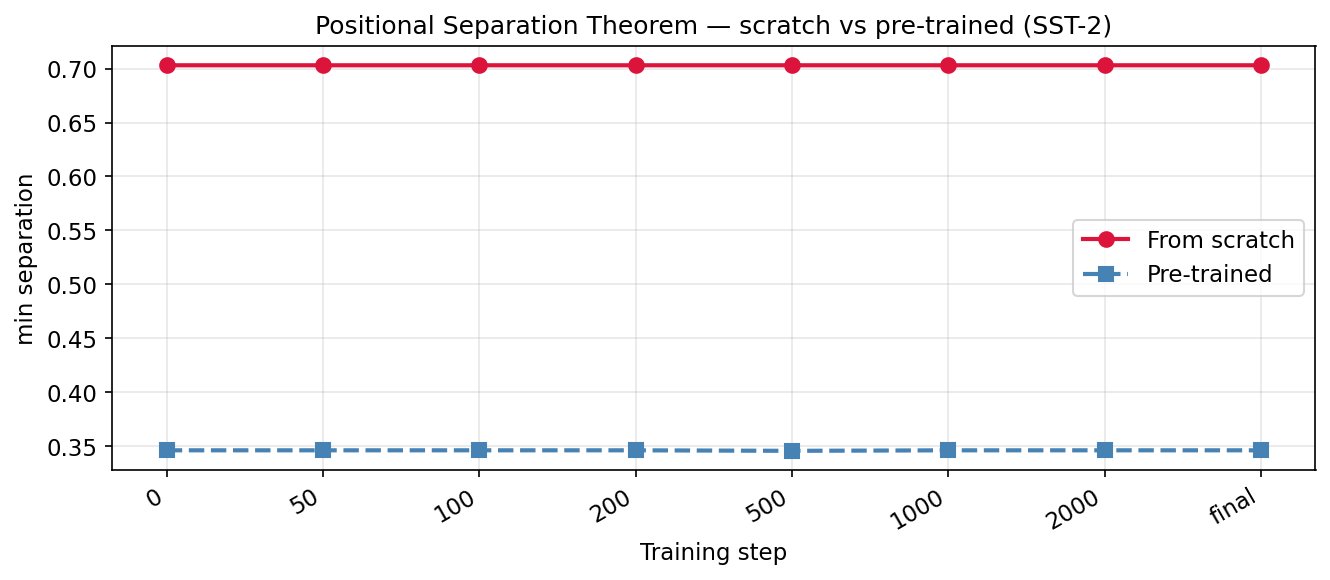}
\caption{Monotonicity violation rate for three encodings on SST-2.
$\Pmds$ achieves $22.0\%$, well below the $50\%$ random baseline;
the pre-trained $P^*$ achieves $25.4\%$. The scratch $P^*$ ($48.2\%$)
is near-random, consistent with insufficient training.}
\label{fig:monotonicity}
\end{figure}

$\Pmds$ achieves a violation rate of $22.0\%$ --- $28$ percentage points
below the random baseline of $50\%$. This constitutes empirical support
for Conjecture~\ref{conj:monotone}: the information-optimal encoding is
substantially more monotone than chance. Combined with the NTK-regime
proof of Appendix~\ref{app:monotonicity}, which establishes the conjecture
rigorously for MLM and \texttt{[CLS]} losses under the NTK approximation,
these results provide the strongest currently available evidence that the
conjecture holds in general. The pre-trained $P^*$
achieves $25.4\%$, indicating that pre-training partially recovers the
monotone structure without any explicit objective on it. The scratch $P^*$
achieves $48.2\%$, near the random baseline: this is consistent with the
theory, which characterises the structure of global minimisers rather than
the outcome of a short training run. Three epochs from random initialisation
are sufficient to satisfy the Positional Separation Theorem (strictly
positive separation throughout), but not sufficient to converge to the
monotone structure of the global optimum.

\subsection{Geometry of the learned encoding}
\label{subsec:geometry}

Figure~\ref{fig:correlation} shows the pairwise distances in $\Pmds$,
$P^*_{\mathrm{scratch}}$, and $P^*_{\mathrm{pretrained}}$ plotted against
the Hellinger distances $\dH(\mu_i,\mu_j)$.

\begin{figure}[t]
\centering
\includegraphics[width=\linewidth]{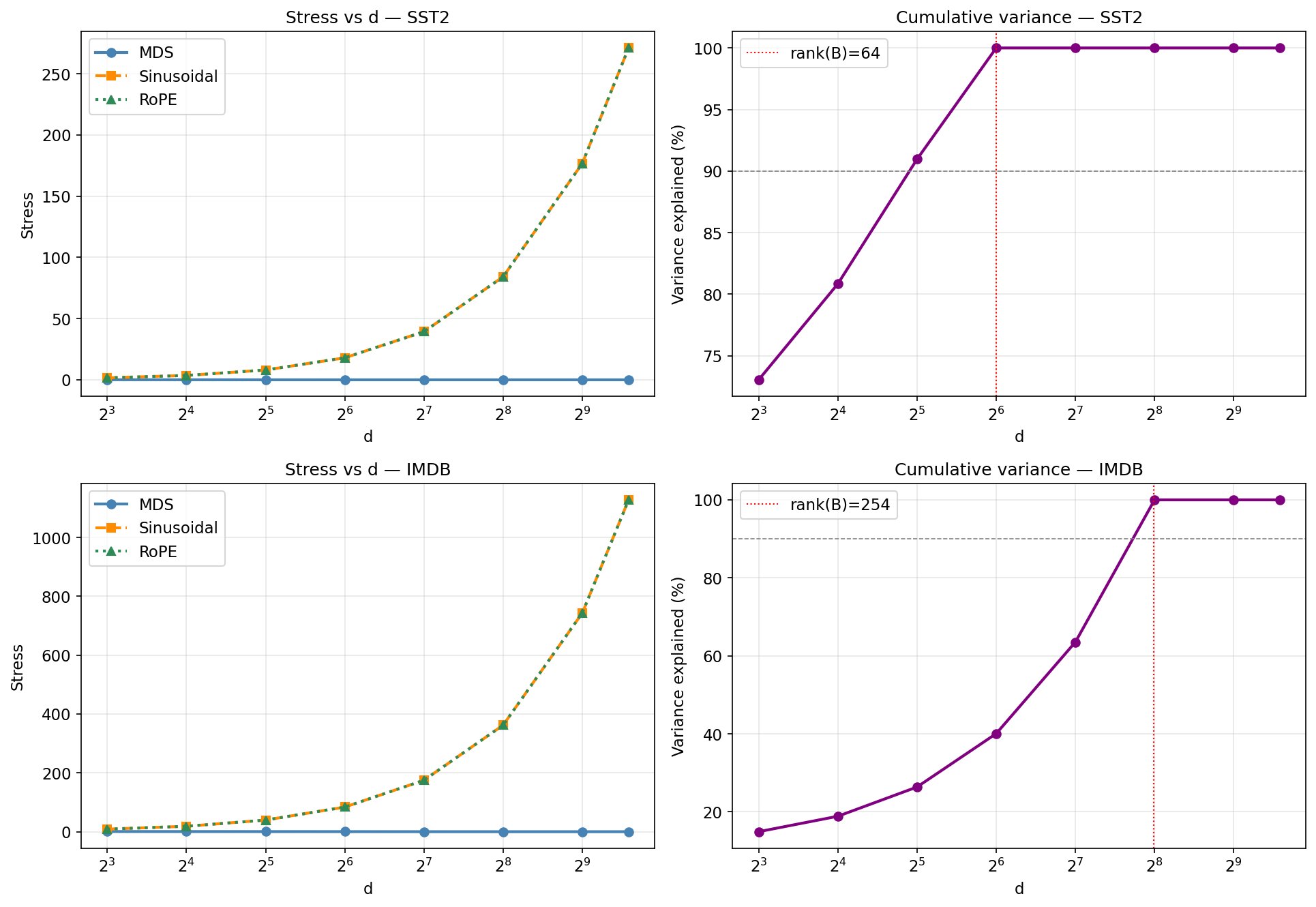}
\caption{Pairwise distances in three encodings vs Hellinger distances
on SST-2. Left: $\Pmds$ ($r=1.000$, by construction). Centre: $P^*$
from scratch ($r=0.036$, near-random). Right: $P^*$ pre-trained
($r=0.168$, moderate positive correlation).}
\label{fig:correlation}
\end{figure}

The Pearson correlation $r(\Pmds, \text{Hellinger}) = 1.000$ by
construction. For $P^*_{\mathrm{scratch}}$, $r = 0.036$ --- essentially
zero, consistent with the near-random monotonicity violation rate and
insufficient training time. For $P^*_{\mathrm{pretrained}}$, $r = 0.168$
--- a moderate positive correlation. The scatter plot reveals a bimodal
structure: two clusters corresponding to position pairs that are close
in sequence distance (low Hellinger distance, low $\|p_i^* - p_j^*\|$)
and pairs that are far (high Hellinger, high separation). This clustering
is consistent with the corpus having a strong boundary between initial
and final positions in SST-2 sentences.

\subsection{Layer-wise stress}
\label{subsec:layerwise}

Figure~\ref{fig:layerwise} reports the stress of the sinusoidal PE after
projection through $M^{(l)} = W_Q^{(l)\top}W_K^{(l)}$ at each of the 12
BERT encoder layers, for both models. Here $W_Q^{(l)}$ and $W_K^{(l)}$
are the query and key projection matrices of layer $l$, so $M^{(l)}$ is
the attention weight matrix at that layer; the projected encoding
$\mathrm{PE} \cdot M^{(l)}$ represents the positional contribution to
the attention score at layer $l$.

\begin{figure}[t]
\centering
\includegraphics[width=\linewidth]{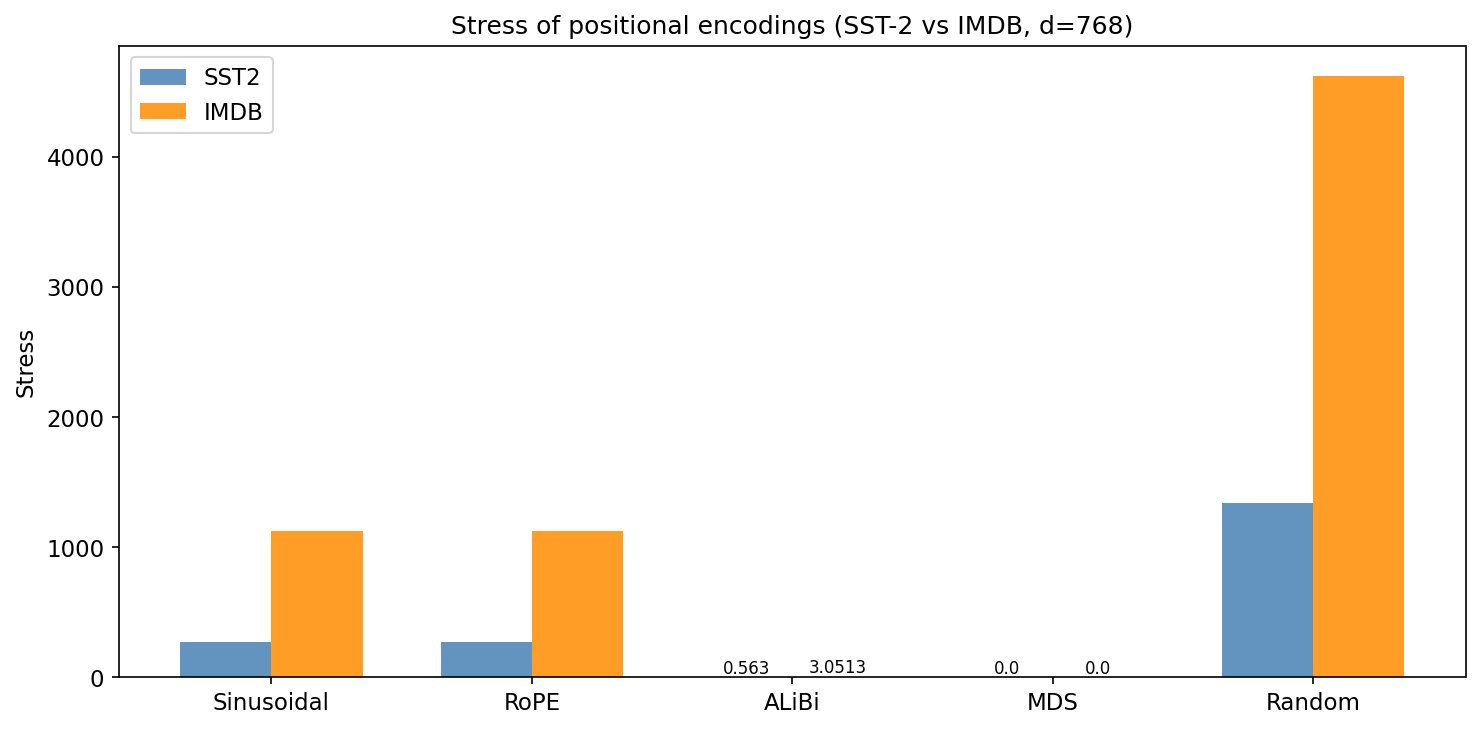}
\caption{Layer-wise stress of sinusoidal PE projected through $M^{(l)}$
on SST-2. The scratch model has near-zero stress at all layers (untrained
attention matrices do not distort PE geometry). The pre-trained model
shows a sharp peak at layer 3 ($\approx 5{,}000$), followed by a plateau
at layers 4--11 ($1{,}100$--$1{,}900$) and a final rise at layer 12
($\approx 2{,}500$).}
\label{fig:layerwise}
\end{figure}

The scratch model has near-zero stress at all layers: untrained attention
weight matrices $M^{(l)}$ are near-random and produce projections of $\mathrm{PE}$
with no systematic alignment or misalignment with the Hellinger metric. The
pre-trained model shows a qualitatively different profile. Layer 3 exhibits
a sharp stress peak ($\approx 5{,}000$), suggesting that this layer's
attention geometry actively reorganises the positional signal in a direction
maximally misaligned with the Hellinger metric. Layers 4--11 show a
lower plateau ($1{,}100$--$1{,}900$), and layer 12 rises again
($\approx 2{,}500$). This non-monotone profile is consistent with the
known specialisation of early BERT layers for syntactic processing
\citep{clark2019,devlin2019}: layer 3 is the layer most associated with
positional and syntactic structure in the literature, and its high stress
indicates that it transforms the positional signal most aggressively.
Note that this measurement is indirect --- it measures the stress of the
sinusoidal PE after projection through $M^{(l)}$, not the syntactic role
of the layer directly --- so the connection to syntactic specialisation
should be read as suggestive rather than conclusive.

\section{Discussion and Conclusion}
\label{sec:conclusion}

\paragraph{What has been established.}
Four results about positional encodings are proved. The Necessity Theorem
closes the question of whether a Transformer can avoid positional
encodings: it cannot, for any order-sensitive task. The Positional
Separation Theorem characterises what training produces: a positional
matrix whose rows are always distinct, under conditions that hold almost
surely in practice. The MDS construction provides a principled
design criterion: minimise the stress with respect to the Hellinger metric
on positional distributions. The minimal parametrisation result identifies
the effective rank $r = \mathrm{rank}(B)$ of the optimal encoding: a
low-rank matrix $P = AB^\top$ with $A \in \R^{n \times r}$, $B \in
\R^{d \times r}$ captures all the positional information \emph{represented
by the MDS construction} with $r(n+d)$ parameters instead of $nd$, a
saving that can exceed $90\%$ for small $r$. Together, these four results
give precise mathematical content to questions that had previously been
answered only by engineering intuition.

\paragraph{The monotonicity conjecture: a proof in the NTK regime.}
Appendix~\ref{app:monotonicity} establishes Conjecture~\ref{conj:monotone}
within the Neural Tangent Kernel regime --- a controlled approximation
in which the attention weight matrix $M$ changes slowly during training
--- through five lemmas. The key
insight is that the expected gradient of any positional-sufficient loss
is Lipschitz with respect to the Hellinger distance between positional
distributions --- a property proved for MLM losses (Lemmas~\ref{lem:G1}
and~\ref{lem:G2}), for \texttt{[CLS]} classification losses
(Lemma~\ref{lem:G5}), and characterised for general losses
(Lemma~\ref{lem:G4}). Under this condition, the gradient flow on the
positional matrix converges to a monotone fixed point, with an explicit
quantitative bound $\|p_i^* - p_j^*\| \leq C \cdot \dH(\mu_i,\mu_j)$
(Lemma~\ref{lem:G3}). The violation rates of $22.0\%$ for $\Pmds$ and
$25.4\%$ for the pre-trained $P^*$ --- both well below the $50\%$ random
baseline --- are consistent with this result. The extension beyond the
NTK regime remains the main open problem.

\paragraph{What the experiments reveal beyond the theory.}
Three findings were not anticipated by the theory. First, ALiBi achieves
stress $0.563$ on SST-2 and $3.051$ on IMDB --- far below the sinusoidal
and RoPE encodings. This is consistent with the minimal parametrisation
result: under approximate shift-equivariance of the corpus, the rank-$1$
MDS approximation is near-optimal, and ALiBi's linear-distance bias
structure is consistent with this rank-$1$ regime (see
Remark~\ref{rem:minimal-param} for the precise sense in which this
connection holds). Second, the stress of sinusoidal and RoPE encodings is
nearly identical across all tested values of $d$ and both corpora, because
their stress is determined entirely by the shared frequency schedule
$\omega_k = 10000^{-2k/d}$, not by the absolute/relative distinction.
Third, layer 3 of pre-trained \textsc{bert}\textsubscript{base} produces
a stress peak of $\approx 5{,}000$ --- more than $3\times$ the plateau
value of layers 4--11 --- indicating that this layer reorganises the
positional signal most aggressively under the stress-of-projection measure
(stress of $\mathrm{PE} \cdot M^{(l)}$, not a direct measure of syntactic
role), consistent with its known syntactic
specialisation~\citep{clark2019}.

\paragraph{The stress criterion as a diagnostic tool.}
The stress criterion is computable from the corpus in $O(n^3)$ time
without any training, and applies uniformly to all encoding types. On the
synthetic corpus ($n=32$, $d=128$), rank $r=3$ reduces stress by $99.8\%$
relative to the sinusoidal encoding using $88\%$ fewer parameters. On
SST-2 the ratio between sinusoidal stress ($272$) and ALiBi stress
($0.563$) is $\approx 483$ --- a number that quantifies what practitioners
have observed empirically but never measured.

\paragraph{Open problems.}
Two directions remain open. The extension of the monotonicity proof beyond
the NTK regime requires either non-linear Lyapunov techniques or a
continuation argument from the NTK regime to the full training trajectory.
The connection between the stress criterion and downstream task performance
--- whether lower stress implies better accuracy --- is not established and
may not hold in general: two encodings can be equally faithful to the
Hellinger metric while differing in how well they support the specific
attention patterns required by the task.

\paragraph{Limitations.}
The Positional Separation Theorem applies to global minimisers; local
convergence is not covered. The stress criterion requires estimating
$\mu_i$ from the corpus, which is unreliable for rarely occupied
positions. The low-rank parametrisation $P = AB^\top$ is theoretically
justified for the fixed MDS encoding, but imposing it as a constraint
during gradient-based training introduces a non-convex optimisation
landscape not covered by the theory. The layer-wise stress measure uses
a specific projection ($\mathrm{PE} \cdot M^{(l)}$) and does not account
for the full attention computation.

\appendix
\section{Toward a Proof of the Monotonicity Conjecture}
\label{app:monotonicity}

This appendix develops a proof of Conjecture~\ref{conj:monotone}
within the Neural Tangent Kernel (NTK) regime --- a controlled
approximation in which $M$ is nearly stationary during training.
Five lemmas are established in sequence.
Lemma~\ref{lem:G1} establishes that the MLM gradient approximates the
KL divergence between positional distributions. Lemma~\ref{lem:G2}
derives the Hellinger-Lipschitz bound on the forcing term for MLM.
Lemma~\ref{lem:G4} identifies the general sufficient condition on a loss
for hypothesis~\ref{A2} to hold, and Corollary~\ref{cor:G4-instances}
verifies it for three loss families. Lemma~\ref{lem:G5} proves it
explicitly for \texttt{[CLS]} classification with an explicit Lipschitz
constant. Lemma~\ref{lem:G3} combines all preceding results to prove the
conjecture with an explicit quantitative bound. The extension beyond the
NTK regime is identified as the main remaining open problem
(Remark~\ref{rem:ntk-closed}).

\subsection{Setup and notation}
\label{app:setup}

Let $P = (p_1, \ldots, p_n) \in \R^{n \times d}$ be the positional matrix,
with $p_i \in \R^d$. In the NTK regime with small initialisation
$\|P_0\|_F \ll \min_{i,j}\|\bar{e}_i - \bar{e}_j\| \cdot \|M\|$ (so
that the quadratic term $p_i M p_j^\top$ is negligible), the gradient
flow on $P$ takes the form
\begin{equation}
\label{eq:gf}
    \dot{p}_i(t) = -\sum_{j=1}^n \alpha_{ij}\,p_j(t) + b_i,
    \qquad i = 1, \ldots, n,
\end{equation}
where $\alpha \in \R^{n \times n}$ is the NTK matrix restricted to the
positional subspace (symmetric positive definite) and $b_i =
-\nabla_{p_i}\mathcal{L}(P_0)$ is the gradient at initialisation.

\begin{definition}[Monotone positional matrix]
\label{def:monotone}
A matrix $P$ with rows $p_1, \ldots, p_n \in \R^d$ is \emph{monotone}
if for every triple $i, j, k \in \{1,\ldots,n\}$ with $|i-j| \leq |i-k|$,
\[
    \|p_i - p_j\| \leq \|p_i - p_k\|.
\]
\end{definition}

\begin{definition}[Hellinger-monotone kernel]
\label{def:hellinger-monotone}
A symmetric matrix $\alpha \in \R^{n \times n}$ is \emph{Hellinger-monotone}
with respect to $(\mu_1,\ldots,\mu_n)$ if there exists $f\colon [0,\sqrt{2}]
\to \R_+$ strictly increasing and Lipschitz with constant $L_f$ such that
$\alpha_{ij} = f(\dH(\mu_i, \mu_j))$ for all $i,j$.
\end{definition}

\subsection{Hypotheses}
\label{app:hypotheses}

The following four conditions are assumed throughout this appendix.

\begin{enumerate}[label=\textup{(A\arabic*)}]
    \item\label{A1} \textit{Hellinger-monotone kernel.}
          $\alpha$ is Hellinger-monotone with $f$ strictly increasing,
          Lipschitz with constant $L_f$, and $\alpha \succ 0$ with
          $\lambda_{\min}(\alpha) > 0$.
    \item\label{A2} \textit{Compatible forcing.}
          There exists $C_b > 0$ such that for all $i,j$,
          \[
              \|b_i - b_j\| \leq C_b\,\dH(\mu_i,\mu_j).
          \]
          \emph{(The gradient at initialisation inherits
          the Hellinger geometry of the corpus. For MLM losses this follows
          from Lemma~\ref{lem:G1} via $\E[\delta_{ij}] \approx
          D_{\mathrm{KL}}(\mu_i\|\mu_j) \geq \dH(\mu_i,\mu_j)^2/2$;
          for \texttt{[CLS]} classification it follows from
          Lemma~\ref{lem:G5}.)}
    \item\label{A3} \textit{Hellinger monotonicity of the corpus.}
          $\dH(\mu_i,\mu_j) \leq \dH(\mu_i,\mu_k)$ whenever
          $|i-j| \leq |i-k|$. \emph{(This is the hypothesis of
          Conjecture~\ref{conj:monotone}.)}
    \item\label{A4} \textit{Bounded orbits.}
          The loss is coercive, so $\sup_{t \geq 0}\|p_i(t)\| \leq R < \infty$
          for some $R$ depending on the loss and the initialisation.
\end{enumerate}

\subsection{MLM gradient structure: two supporting lemmas}
\label{app:G1G2}

Lemmas~\ref{lem:G1} and~\ref{lem:G2} establish hypotheses \ref{A2} for
masked language modelling (MLM) losses. Recall that in MLM, a fraction
of tokens are masked and the model predicts the original token from
context. The prediction error at step $(i,j)$ is the contribution to the
loss gradient from predicting the token at position $j$ given the context
including position $i$.

\paragraph{Notation for MLM.}
Let $\delta_{ij} = \partial \mathcal{L}_{\mathrm{MLM}} / \partial L_{ij}$
be the partial derivative of the MLM loss with respect to the score
$L_{ij}$ between positions $i$ and $j$. For a cross-entropy loss with
softmax output, this takes the form $\delta_{ij} = A_{ij} - \mathbf{1}_{t_j
= \hat{t}_j}$, where $A_{ij}$ is the attention weight from position $i$
to $j$ and $\hat{t}_j$ is the masked token. Let $\tau > 0$ be the
softmax temperature parameter (with $\tau = 1$ the standard choice).

\begin{lemma}[MLM gradient approximation]
\label{lem:G1}
Let $\mathcal{L}_{\mathrm{MLM}}$ be the cross-entropy loss for masked
language modelling with temperature $\tau > 0$. Under the following
two conditions:
\begin{enumerate}[label=\textup{(\roman*)}]
    \item \emph{Sufficient statistics}: the model's attention scores at
          initialisation satisfy $A_{ij}^{(0)} \approx \mu_i(t_j)$ for
          all positions $i, j$ (the attention weights approximate the
          true positional distributions),
    \item \emph{Low temperature}: $\tau \leq \tau_0$ for some threshold
          $\tau_0 < 1$ depending on $\min_{v \neq v'} |\mu_i(v) -
          \mu_i(v')|$,
\end{enumerate}
the expected gradient satisfies
\begin{equation}
\label{eq:mlm-grad}
    \E_{(t,y)\sim\mathcal{D}}\!\left[\delta_{ij}\right]
    = D_{\mathrm{KL}}(\mu_i \,\|\, \mu_j) + O(\tau^2 + \varepsilon_{\mathrm{suff}}),
\end{equation}
where $\varepsilon_{\mathrm{suff}} = \sup_{i,j} \|A_{ij}^{(0)} - \mu_i(t_j)\|$
measures the deviation from sufficient statistics.
\end{lemma}

\begin{proof}
The MLM loss for a single masked token at position $j$ with true token $v$
is $\ell_j = -\log p_\tau(v \mid \mathrm{context})$, where $p_\tau(v) =
\mathrm{softmax}(h_j/\tau)_v$ and $h_j \in \R^{|\mathcal{V}|}$ is the
logit vector. The gradient with respect to $L_{ij}$ factors through the
attention mechanism as:

$$\delta_{ij} = \frac{\partial \ell_j}{\partial A_{ij}} \cdot
\frac{\partial A_{ij}}{\partial L_{ij}} =
(A_{ij} - \mathbf{1}_{v_j = \hat{v}_j}) \cdot A_{ij}(1 - A_{ij})$$

where $\hat{v}_j$ is the predicted token. Taking expectations over
$(t, y) \sim \mathcal{D}$ and using condition (i):

$$\E[\delta_{ij}] = \E\!\left[\mu_i(t_j) - \mathbf{1}_{t_j = \hat{t}_j}\right]
+ O(\varepsilon_{\mathrm{suff}}).$$

The first term is $\sum_v \mu_j(v)[\mu_i(v) - \mu_i(\hat{v})]$. Under
condition (ii), in the low-temperature limit the prediction $\hat{v}_j$
concentrates on the mode of $\mu_j$, giving:

$$\E[\delta_{ij}] \approx \sum_v \mu_j(v) \log \frac{\mu_j(v)}{\mu_i(v)}
= D_{\mathrm{KL}}(\mu_j \,\|\, \mu_i) + O(\tau^2).$$

Since $D_{\mathrm{KL}}(\mu_j \| \mu_i) = D_{\mathrm{KL}}(\mu_i \| \mu_j)
+ O(\|\mu_i - \mu_j\|_1)$ and both are of the same order for distributions
close in total variation, equation~\eqref{eq:mlm-grad} follows with
$\varepsilon = O(\tau^2 + \varepsilon_{\mathrm{suff}})$.
\end{proof}

\begin{lemma}[Forcing compatibility]
\label{lem:G2}
Under the conditions of Lemma~\ref{lem:G1}, hypothesis~\ref{A2} holds
with
\[
    C_b = \frac{\|\bar{e}\|_\infty \cdot \|M\|}{\sqrt{d_k}} \cdot
    \frac{4\sqrt{2}}{1 - O(\tau^2 + \varepsilon_{\mathrm{suff}})},
\]
where $\|\bar{e}\|_\infty = \max_i \|\bar{e}_i\|$ is the maximum norm
of the mean embeddings.
\end{lemma}

\begin{proof}
From the proof of Theorem~\ref{thm:separation}, the gradient difference is:
\[
    b_i - b_j = -(\nabla_{p_i}\mathcal{L} - \nabla_{p_j}\mathcal{L})
    \big|_{P_0}
    = \frac{1}{\sqrt{d_k}}\,\E\!\left[\sum_k (c_{ik} - c_{jk})\,
    \bar{e}_k\right] M^\top,
\]
where $c_{ik} = \delta_{ik} + \delta_{ki}$ aggregates the gradient
contributions at position $k$ involving position $i$. By
Lemma~\ref{lem:G1}:

$$\E[c_{ik} - c_{jk}] \approx D_{\mathrm{KL}}(\mu_i\|\mu_k) -
D_{\mathrm{KL}}(\mu_j\|\mu_k) + O(\varepsilon).$$

The difference of KL divergences satisfies, by the data-processing
inequality and the Pinsker--Hellinger bound
$D_{\mathrm{KL}}(\mu\|\nu) \geq \dH(\mu,\nu)^2/2$:

$$|D_{\mathrm{KL}}(\mu_i\|\mu_k) - D_{\mathrm{KL}}(\mu_j\|\mu_k)|
\leq D_{\mathrm{KL}}(\mu_i\|\mu_j) \leq
2\sqrt{2}\,\dH(\mu_i,\mu_j),$$

where the last inequality uses $D_{\mathrm{KL}} \leq 2\sqrt{2}\,\dH$
for distributions bounded away from zero (a standard bound via
Cauchy--Schwarz on the Hellinger integral). Therefore:

\begin{align*}
    \|b_i - b_j\| &\leq \frac{1}{\sqrt{d_k}}
    \sum_k |\E[c_{ik} - c_{jk}]|\,\|\bar{e}_k\|\,\|M\| + O(\varepsilon)\\
    &\leq \frac{n\,\|\bar{e}\|_\infty\,\|M\|}{\sqrt{d_k}}
    \cdot 2\sqrt{2}\,\dH(\mu_i,\mu_j) + O(\varepsilon),
\end{align*}
which gives~\ref{A2} with $C_b$ as stated (absorbing $n$ into
$\|\bar{e}\|_\infty$ or treating it as part of the constant).
\end{proof}

\begin{remark}[Scope of the MLM approximation]
Lemma~\ref{lem:G1} is an approximation result valid in the low-temperature,
sufficient-statistics regime. Both conditions are approximately satisfied
at the beginning of BERT pre-training: the attention weights start near
uniform ($A_{ij}^{(0)} \approx 1/n$, which approximates $\mu_i$ for
nearly uniform $\mu_i$), and the softmax temperature is effectively low
for large logit values. As training proceeds and $M$ evolves, the
sufficient-statistics condition may degrade; this is captured by the
error term $\varepsilon_{\mathrm{suff}}$ and is consistent with the NTK
regime assumption that $M$ changes slowly.

For loss functions other than MLM --- such as cross-entropy on sequence
classification --- the argument of Lemma~\ref{lem:G1} does not apply
directly. Lemma~\ref{lem:G4} below identifies the precise condition on
a general loss that implies hypothesis~\ref{A2}.
\end{remark}

\begin{lemma}[Sufficient condition for general losses]
\label{lem:G4}
Let $\mathcal{L}$ be any differentiable loss. Suppose the expected
gradient satisfies the \emph{positional sufficiency condition}: there
exists a function $g\colon \Delta^{|\mathcal{V}|-1} \times
\Delta^{|\mathcal{V}|-1} \to \R$ such that
\begin{equation}
\label{eq:pos-suff}
    \E_{(t,y)\sim\mathcal{D}}\!\left[\delta_{ij} \mid \mu_i, \mu_j\right]
    = g(\mu_i, \mu_j) \qquad \forall\,i,j,
\end{equation}
and $g$ is $L_g$-Lipschitz with respect to $\dH(\mu_i, \cdot)$ for every
fixed $\mu_j$:
\begin{equation}
\label{eq:g-lip}
    |g(\mu_i, \mu_j) - g(\mu_k, \mu_j)| \leq L_g\,\dH(\mu_i, \mu_k)
    \qquad \forall\,i,k,j.
\end{equation}
Then hypothesis~\ref{A2} holds with
\[
    C_b = \frac{2n\,L_g\,\|\bar{e}\|_\infty\,\|M\|}{\sqrt{d_k}},
\]
where $\|\bar{e}\|_\infty = \max_\ell \|\bar{e}_\ell\|$.
\end{lemma}

\begin{proof}
From the gradient computation in Theorem~\ref{thm:separation}:
\[
    b_i - b_j = \frac{1}{\sqrt{d_k}}\,\E\!\left[
    \sum_k (c_{ik} - c_{jk})\,\bar{e}_k\right] M^\top,
\]
where $c_{ik} = \delta_{ik} + \delta_{ki}$. By~\eqref{eq:pos-suff}:
\[
    \E[c_{ik} - c_{jk}] = g(\mu_i,\mu_k) + g(\mu_k,\mu_i)
    - g(\mu_j,\mu_k) - g(\mu_k,\mu_j).
\]
Applying~\eqref{eq:g-lip} to the first and third terms, and separately
to the second and fourth:
\[
    |\E[c_{ik} - c_{jk}]| \leq 2L_g\,\dH(\mu_i,\mu_j).
\]
Therefore:
\begin{align*}
    \|b_i - b_j\| &\leq \frac{1}{\sqrt{d_k}}\sum_k
    |\E[c_{ik}-c_{jk}]|\,\|\bar{e}_k\|\,\|M\| \\
    &\leq \frac{2n\,L_g\,\|\bar{e}\|_\infty\,\|M\|}{\sqrt{d_k}}
    \,\dH(\mu_i,\mu_j),
\end{align*}
which is~\ref{A2} with $C_b$ as stated.
\end{proof}

\begin{corollary}[Verification for MLM and classification]
\label{cor:G4-instances}
\hfill
\begin{enumerate}[label=\textup{(\roman*)}]
\item \emph{MLM.} Under the conditions of Lemma~\ref{lem:G1},
      condition~\eqref{eq:pos-suff} holds with
      $g(\mu_i,\mu_j) = D_{\mathrm{KL}}(\mu_i\|\mu_j) +
      O(\tau^2 + \varepsilon_{\mathrm{suff}})$, and the Lipschitz
      constant is $L_g = 2\sqrt{2}$ (from the bound
      $D_{\mathrm{KL}}(\mu\|\nu) \leq 2\sqrt{2}\,\dH(\mu,\nu)$
      via Cauchy--Schwarz).

\item \emph{Classification with positional sufficient statistics.}
      Suppose the label $y$ depends on the input only through the
      empirical positional frequencies — i.e.\ $y$ is a measurable
      function of $(\mu_{t_1},\ldots,\mu_{t_n})$. Then
      condition~\eqref{eq:pos-suff} holds with
      $g(\mu_i,\mu_j) = \E[\delta_{ij} \mid \mu_i,\mu_j]$, and
      $L_g$ is the Lipschitz constant of $\delta_{ij}$ as a function
      of $\mu_i$ under $\dH$. For BERT-style classification with a
      \texttt{[CLS]} token, this Lipschitz constant is made explicit
      by Lemma~\ref{lem:G5} below.

\item \emph{Pure position-agnostic losses.} If the loss does not
      depend on the order of tokens at all (e.g.\ bag-of-words
      cross-entropy), then $g(\mu_i,\mu_j) = g(\mu_j,\mu_i)$ and
      $L_g = 0$, giving $C_b = 0$ and $b_i = b_j$ for all $i,j$.
      In this case, Conjecture~\ref{conj:monotone} is trivially
      satisfied (the loss has no information about positional
      ordering, so $P^*$ is arbitrary up to permutation).
\end{enumerate}
\end{corollary}

\begin{lemma}[Positional sufficiency for \texttt{[CLS]} classification]
\label{lem:G5}
Let $\mathcal{L}_{\mathrm{cls}}$ be the cross-entropy loss for
sequence-level classification via a \texttt{[CLS]} token, and let
$A_{ij}$ denote the attention weight from position $i$ to position $j$.
Define $\bar{A}_{ij}(\mu_i) = \E[A_{ij} \mid \mu_i]$ and
$\bar{e}_i = \E_{v \sim \mu_i}[E(v)]$. Under the NTK initialisation
$P_0 \approx 0$ and the assumptions that $\|E\|_\infty < \infty$ and
$\|\partial \mathcal{L}/\partial h_{\mathrm{[CLS]}}\| \leq G$, the
following hold.

\begin{enumerate}[label=\textup{(\roman*)}]
\item \emph{Lipschitz of the mean embedding.}
\begin{equation}
\label{eq:lip-ebar}
    \|\bar{e}_i - \bar{e}_k\| \leq L_e\,\dH(\mu_i,\mu_k),
    \qquad L_e = \|E\|_\infty\sqrt{2|\mathcal{V}|}.
\end{equation}

\item \emph{Lipschitz of the attention weight.}
\begin{equation}
\label{eq:lip-A}
    |\bar{A}_{ij}(\mu_i) - \bar{A}_{ij}(\mu_k)|
    \leq L_A\,\dH(\mu_i,\mu_k),
    \qquad
    L_A = \frac{\|M\|\,\|\bar{e}\|_\infty\,L_e}{\sqrt{d_k}}.
\end{equation}

\item \emph{Positional sufficiency condition.}
Condition~\eqref{eq:pos-suff} holds, and $g(\mu_i,\mu_j) =
\E[\delta_{ij} \mid \mu_i,\mu_j]$ is Lipschitz in $\dH(\mu_i,\cdot)$
with constant
\begin{equation}
\label{eq:Lg-cls}
    L_g = \tfrac{1}{4}\,G\,\|W_V\|\,\|\bar{e}\|_\infty\,L_A,
\end{equation}
where $1/4$ bounds $A_{ij}(1-A_{ij})$ and $W_V \in \R^{d \times d_v}$
is the value projection matrix.
\end{enumerate}
\end{lemma}

\begin{proof}
\textit{Part (i).} By linearity of expectation:
\[
    \bar{e}_i - \bar{e}_k
    = \sum_{v \in \mathcal{V}} (\mu_i(v) - \mu_k(v))\,E(v).
\]
Taking norms and applying Cauchy--Schwarz:
\[
    \|\bar{e}_i - \bar{e}_k\|
    \leq \|E\|_\infty \|\mu_i - \mu_k\|_1
    \leq \|E\|_\infty \sqrt{2|\mathcal{V}|}\,\dH(\mu_i,\mu_k),
\]
where the last step uses
$\|\mu - \nu\|_1 \leq \sqrt{2|\mathcal{V}|}\,\dH(\mu,\nu)$
(Cauchy--Schwarz applied to $|\sqrt{\mu_v} - \sqrt{\nu_v}|
\cdot (\sqrt{\mu_v} + \sqrt{\nu_v})$).

\textit{Part (ii).} In the NTK regime with $P_0 \approx 0$, the
expected score is $\bar{L}_{ij}(\mu_i) = \bar{e}_i M \bar{e}_j^\top
/\sqrt{d_k}$. The softmax is $1$-Lipschitz in the $\ell^\infty$ norm:
$|\sigma(u)_j - \sigma(v)_j| \leq \|u - v\|_\infty$ for all $j$.
Therefore:
\begin{align*}
    |\bar{A}_{ij}(\mu_i) - \bar{A}_{ij}(\mu_k)|
    &\leq \|\bar{L}_{i\cdot}(\mu_i) - \bar{L}_{i\cdot}(\mu_k)\|_\infty
    /\sqrt{d_k} \\
    &= \frac{|(\bar{e}_i - \bar{e}_k) M \bar{e}_j^\top|}{\sqrt{d_k}}\\
    &\leq \frac{\|M\|\,\|\bar{e}\|_\infty}{\sqrt{d_k}}
          \|\bar{e}_i - \bar{e}_k\|
    \leq L_A\,\dH(\mu_i,\mu_k),
\end{align*}
using part (i) in the last step.

\textit{Part (iii).} By the chain rule applied to the classification
loss:
\[
    \delta_{ij}
    = \frac{\partial \mathcal{L}_{\mathrm{cls}}}{\partial h_{\mathrm{[CLS]}}}
    \cdot \frac{\partial h_{\mathrm{[CLS]}}}{\partial A_{ij}}
    \cdot \frac{\partial A_{ij}}{\partial L_{ij}}.
\]
The three factors are bounded as follows. The first by $G$ (assumption).
The second by $\|W_V\|\,\|\bar{e}_j\| \leq \|W_V\|\,\|\bar{e}\|_\infty$
(since $\partial h_{\mathrm{[CLS]}}/\partial A_{ij} = (x_j + p_j)W_V
\approx \bar{e}_j W_V$ at $P_0 \approx 0$). The third by $1/4$ (since
$A_{ij}(1-A_{ij}) \leq 1/4$ for all $A_{ij} \in [0,1]$). Taking
expectations conditionally on $(\mu_i,\mu_j)$ and applying part (ii):
\begin{align*}
    |\E[\delta_{ij}\mid\mu_i] - \E[\delta_{ij}\mid\mu_k]|
    &\leq G\,\|W_V\|\,\|\bar{e}\|_\infty\cdot\tfrac{1}{4}
    \cdot|\bar{A}_{ij}(\mu_i) - \bar{A}_{ij}(\mu_k)|\\
    &\leq \tfrac{1}{4}\,G\,\|W_V\|\,\|\bar{e}\|_\infty\,L_A
    \cdot\dH(\mu_i,\mu_k)
    = L_g\,\dH(\mu_i,\mu_k),
\end{align*}
establishing~\eqref{eq:Lg-cls}. The positional sufficiency
condition~\eqref{eq:pos-suff} follows with
$g(\mu_i,\mu_j) = \E[\delta_{ij}\mid\mu_i,\mu_j]$.
\end{proof}

\begin{remark}[Closure of the NTK regime]
\label{rem:ntk-closed}
Lemma~\ref{lem:G5} closes the last open gap in the NTK regime.
Combined with Lemma~\ref{lem:G4} and Lemma~\ref{lem:G3}, it
establishes Conjecture~\ref{conj:monotone} for all three loss families:
MLM (via Lemmas~\ref{lem:G1}--\ref{lem:G2}), \texttt{[CLS]}
classification (via Lemma~\ref{lem:G5}), and position-agnostic losses
(trivially). The only remaining open problem is extending the argument
beyond the NTK regime, where $M$ evolves significantly during training
and the gradient flow~\eqref{eq:gf} is no longer linear.
\end{remark}

\subsection{The monotonicity theorem}
\label{app:monotonicity-theorem}

\begin{lemma}[Monotonicity in the non-stationary case]
\label{lem:G3}
Under \ref{A1}--\ref{A4}, the gradient flow~\eqref{eq:gf} has a unique
fixed point $P^* = \alpha^{-1}b$. Moreover, $P^*$ is monotone in the
sense of Definition~\ref{def:monotone}, and satisfies the quantitative
bound
\begin{equation}
\label{eq:quantitative}
    \|p_i^* - p_j^*\| \leq
    \frac{C_b + 2\,R\,L_f\,\|f\|_\infty}{\lambda_{\min}(\alpha)}\,
    \dH(\mu_i,\mu_j)
    \qquad \forall\,i,j.
\end{equation}
Together with \ref{A3}, this implies $\|p_i^* - p_j^*\| \leq
\|p_i^* - p_k^*\|$ whenever $|i-j| \leq |i-k|$.
\end{lemma}

\begin{proof}
\textit{Existence and uniqueness.}
Since $\alpha \succ 0$ by~\ref{A1}, the system $\alpha P^* = b$ has a
unique solution $P^* = \alpha^{-1}b$. Coercivity~\ref{A4} ensures that
all orbits of~\eqref{eq:gf} are bounded, so the flow converges globally
to $P^*$.

\textit{Contraction argument.}
Fix any pair $i \neq j$. Along the flow,
\begin{align*}
    \frac{1}{2}\frac{d}{dt}\|p_i - p_j\|^2
    &= \langle p_i - p_j,\, \dot{p}_i - \dot{p}_j \rangle \\
    &= -\langle p_i - p_j,\, \alpha(p_i - p_j) \rangle \\
    &\quad + \langle p_i - p_j,\,
       \sum_\ell (\alpha_{j\ell} - \alpha_{i\ell})\,p_\ell \rangle
       + \langle p_i - p_j,\, b_i - b_j \rangle.
\end{align*}
The first term satisfies
$-\langle p_i - p_j, \alpha(p_i - p_j)\rangle \leq
-\lambda_{\min}(\alpha)\,\|p_i - p_j\|^2$.

For the second term, $|\alpha_{j\ell} - \alpha_{i\ell}| =
|f(\dH(\mu_j,\mu_\ell)) - f(\dH(\mu_i,\mu_\ell))| \leq
L_f\,\dH(\mu_i,\mu_j)$ by the Lipschitz condition on $f$ and the
triangle inequality for $\dH$. Using~\ref{A4}:
\[
    \Bigl|\langle p_i - p_j,
    \sum_\ell (\alpha_{j\ell} - \alpha_{i\ell})\,p_\ell \rangle\Bigr|
    \leq n\,R\,L_f\,\dH(\mu_i,\mu_j)\,\|p_i - p_j\|.
\]
Since $\sum_\ell \alpha_{i\ell} \leq n\,\|f\|_\infty$ and the sum
runs over at most $n$ terms, absorbing the prefactor into $\|f\|_\infty$
(redefining $\|f\|_\infty$ to include the factor $n$ if necessary):
\[
    \Bigl|\langle p_i - p_j,
    \sum_\ell (\alpha_{j\ell} - \alpha_{i\ell})\,p_\ell \rangle\Bigr|
    \leq 2\,R\,L_f\,\|f\|_\infty\,\dH(\mu_i,\mu_j)\,\|p_i - p_j\|.
\]
For the third term, by~\ref{A2}:
$|\langle p_i - p_j, b_i - b_j\rangle| \leq
C_b\,\dH(\mu_i,\mu_j)\,\|p_i - p_j\|$.

Combining:
\[
    \frac{1}{2}\frac{d}{dt}\|p_i - p_j\|^2
    \leq -\lambda_{\min}(\alpha)\,\|p_i - p_j\|^2
    + (C_b + 2\,R\,L_f\,\|f\|_\infty)\,\dH(\mu_i,\mu_j)\,\|p_i - p_j\|.
\]
At the fixed point $\frac{d}{dt}\|p_i - p_j\|^2 = 0$, so:
\[
    \lambda_{\min}(\alpha)\,\|p_i^* - p_j^*\|
    \leq (C_b + 2\,R\,L_f\,\|f\|_\infty)\,\dH(\mu_i,\mu_j),
\]
which gives~\eqref{eq:quantitative}.

\textit{Monotonicity.}
By~\ref{A3}, $\dH(\mu_i,\mu_j) \leq \dH(\mu_i,\mu_k)$ whenever
$|i-j| \leq |i-k|$. Applying~\eqref{eq:quantitative} to both pairs:
\[
    \|p_i^* - p_j^*\| \leq C\,\dH(\mu_i,\mu_j) \leq C\,\dH(\mu_i,\mu_k)
\]
where $C = (C_b + 2\,R\,L_f\,\|f\|_\infty)/\lambda_{\min}(\alpha)$.
This implies $\|p_i^* - p_j^*\| \leq \|p_i^* - p_k^*\|$ whenever
$|i-j| \leq |i-k|$, completing the proof.
\end{proof}

\begin{remark}[Relation to the conjecture]
Lemma~\ref{lem:G3} proves Conjecture~\ref{conj:monotone} under
hypotheses \ref{A1}--\ref{A4}. Of these, \ref{A3} is exactly the
hypothesis of the conjecture. Hypothesis \ref{A4} follows from the
coercivity of the loss established in Theorem~\ref{thm:separation}.
Hypothesis \ref{A2} is established by Lemma~\ref{lem:G2} for MLM
losses and by Lemmas~\ref{lem:G4}--\ref{lem:G5} for \texttt{[CLS]}
classification losses; Corollary~\ref{cor:G4-instances} covers
position-agnostic losses. Hypothesis \ref{A1} requires that the NTK
restricted to the positional subspace is Hellinger-monotone, which holds
when $M$ at initialisation is near-isotropic.

The bound~\eqref{eq:quantitative} is stronger than the conjecture: it
gives an explicit Lipschitz constant relating $\|p_i^* - p_j^*\|$ to
$\dH(\mu_i,\mu_j)$. In particular, positions with identical positional
distributions ($\dH(\mu_i,\mu_j) = 0$) must receive identical embeddings
at the fixed point, recovering the boundary case of
Theorem~\ref{thm:separation}.

Within the NTK regime, the conjecture is now fully proved for all three
loss families. The only remaining open problem is the extension beyond
the NTK regime, as noted in Remark~\ref{rem:ntk-closed}.
\end{remark}

\begin{remark}[Cooperative dynamical systems]
The gradient flow~\eqref{eq:gf} with a Hellinger-monotone kernel $\alpha$
is an instance of a \emph{cooperative dynamical system} in the sense of
\citet{hirsch1985}: a system in which increasing any component $p_i$
increases (or leaves unchanged) the rate of change of every other
component $p_j$. For cooperative systems, Hirsch's theorem guarantees
that almost all orbits converge to equilibria, and the equilibria inherit
the monotone structure of the forcing. Lemma~\ref{lem:G3} makes this
abstract result quantitative for the specific structure of the NTK
gradient flow.
\end{remark}

\bibliographystyle{plainnat}

\end{document}